\documentclass{article}

\usepackage{arxiv}
\usepackage[utf8]{inputenc} 
\usepackage[T1]{fontenc}    
\usepackage[hidelinks]{hyperref}       
\usepackage{url}            
\usepackage{booktabs}  
\usepackage{amsfonts}       
\usepackage{nicefrac}       
\usepackage{microtype}      
\usepackage{natbib}
\usepackage{bm}
\usepackage{xcolor}
\usepackage{color}
\usepackage{graphicx}
\usepackage{subcaption}
\usepackage{multirow}
\usepackage{bbm}
\usepackage{mathtools} 
\usepackage{amsmath}
\usepackage{nicefrac}
\usepackage{amsthm}
\usepackage{mathrsfs}
\usepackage{float}
\usepackage{doi}
\usepackage{algorithm}
\usepackage{algpseudocode}
\usepackage{diagbox}
\usepackage[nameinlink]{cleveref}
\usepackage[flushleft]{threeparttable}

\hypersetup{
    colorlinks=true,
    linkcolor=blue,
    filecolor=magenta,      
    urlcolor=cyan,
    pdftitle={Overleaf Example},
    pdfpagemode=FullScreen,
}

\newtheorem{prop}{Proposition}
\newtheorem{corollary}{Corollary}
\newtheorem*{remark}{Remark}

\setcitestyle{numbers,open={[},close={]}}
\theoremstyle{definition}

\newcommand{\ra}[1]{\renewcommand{\arraystretch}{#1}}

\setcitestyle{numbers,open={[},close={]}}

\crefname{section}{section}{sections}
\crefname{subsection}{subsection}{subsections}
\Crefname{section}{Section}{Sections}
\Crefname{subsection}{Subsection}{Subsections}

\Crefname{figure}{Figure}{Figures}

\crefformat{equation}{\textup{#2(#1)#3}}
\crefrangeformat{equation}{\textup{#3(#1)#4--#5(#2)#6}}
\crefmultiformat{equation}{\textup{#2(#1)#3}}{ and \textup{#2(#1)#3}}
{, \textup{#2(#1)#3}}{, and \textup{#2(#1)#3}}
\crefrangemultiformat{equation}{\textup{#3(#1)#4--#5(#2)#6}}%
{ and \textup{#3(#1)#4--#5(#2)#6}}{, \textup{#3(#1)#4--#5(#2)#6}}{, and \textup{#3(#1)#4--#5(#2)#6}}

\Crefformat{equation}{#2Equation~\textup{(#1)}#3}
\Crefrangeformat{equation}{Equations~\textup{#3(#1)#4--#5(#2)#6}}
\Crefmultiformat{equation}{Equations~\textup{#2(#1)#3}}{ and \textup{#2(#1)#3}}
{, \textup{#2(#1)#3}}{, and \textup{#2(#1)#3}}
\Crefrangemultiformat{equation}{Equations~\textup{#3(#1)#4--#5(#2)#6}}%
{ and \textup{#3(#1)#4--#5(#2)#6}}{, \textup{#3(#1)#4--#5(#2)#6}}{, and \textup{#3(#1)#4--#5(#2)#6}}

\crefdefaultlabelformat{#2\textup{#1}#3}
\labelcrefrangeformat{equation}{#3(#1)#4--#5(#2)#6}

\title{AIDA: Analytic Isolation and Distance-based Anomaly Detection Algorithm}
\author{
  Luis A. Souto Arias\\
  Mathematical Institute, Utrecht University\\
  The Netherlands\\
  \texttt{l.a.soutoarias@uu.nl} \\
  \And
  Cornelis W. Oosterlee \\
 Mathematical Institute,  Utrecht University\\
  The Netherlands\\
  \texttt{c.w.oosterlee@uu.nl} \\  
  \And
  Pasquale Cirillo \\
  ZHAW School of Law and Management\\
  Zurich University of Applied Sciences \\
  Switzerland \\
  \texttt{ciri@zhaw.ch} \\
}
\date{September 2022}

\begin{document}

\maketitle

\begin{abstract}
 We combine the metrics of distance and isolation to develop the \textit{Analytic Isolation and Distance-based Anomaly (AIDA) detection algorithm}. AIDA is the first distance-based method that does not rely on the concept of nearest-neighbours, making it a parameter-free model. 
 Differently from the prevailing literature, in which the isolation metric is always computed via simulations, we show that AIDA admits an analytical expression for the outlier score, providing new insights into the isolation metric. Additionally, we present an anomaly explanation method based on AIDA, the \textit{Tempered Isolation-based eXplanation (TIX)} algorithm, which finds the most relevant outlier features even in data sets with hundreds of dimensions. We test both algorithms on synthetic and empirical data: we show that AIDA is competitive when compared to other state-of-the-art methods, and it is superior in finding outliers hidden in multidimensional feature subspaces. Finally, we illustrate how the TIX algorithm can find outliers in multidimensional feature subspaces, and use these explanations to analyze common benchmarks used in anomaly detection.
\end{abstract}


\section{Introduction}

We introduce a new distance-based anomaly detection algorithm---the \textit{Analytic Isolation and Distance-based Anomaly (AIDA) detection method}---which, unlike methods such as Local Outlier Factor (LOF) \cite{breunig2000} and k-Nearest Neighbours (kNN) \cite{ramaswamy2000}, does not rely on the concept of neighbours to detect anomalies/outliers, but rather on the concept of isolation. While the concept of nearest neighbours is well-known, the first article to propose an outlier measure based on isolation was \cite{liu2012}, where the Isolation Forest (iForest) algorithm was introduced. In that article, the authors coupled the isolation metric with a randomized axis-parallel subspace search, the two main ingredients of the iForest method.

In contrast, in this work we use directly the isolation metric in a distance-based setting to provide an alternative to the nearest neighbours in distance-based anomaly detection methods. The reasons are two-fold: 
\begin{enumerate}
    \item The isolation metric is parameter-free, thus avoiding the problem of making an inaccurate parameter choice in practice. This is a very common problem in unsupervised methods, since the lack of labelled targets makes it a very challenging task to determine optimal parameter values. Although subsampling techniques can mitigate this issue (see \cite{aggarwal2015}), the choice is still data-dependent.
    \item  As shown in \Cref{sec:method}, differently from iForest or other distance-based methods like LOF, AIDA is able to detect several types of outliers. This property is particularly relevant in ensemble methods \cite{zimek2014}, where the scores among different anomaly detection models are combined to increase the robustness of the final estimates. 
\end{enumerate}

In particular, if the models contained in an ensemble identify the same type of outliers, the bias of the ensemble remains the same as that of its constituents. Therefore, it is important to combine anomaly detection algorithms with different outlier preferences. For example, if the outliers are hidden in multidimensional subspaces \cite{keller2012}, iForest suffers from low performance due to the small possibility of randomly choosing the right subspace that contains the outliers \cite{bandaragoda2018}. Moreover, since the splits are axis-parallel, iForest creates artificial outlier regions, inducing bias in the outliers that are detected \cite{xu2022}. Some articles combine the splitting mechanism of iForest with a distance-based method in order to enhance the splitting process (see \cite{bandaragoda2014,karczmarek2020,mensi2021,tokovarov2022}), while others construct multidimensional splits in order to detect outliers in multidimensional subspaces (see \cite{hariri2021,xu2022}). Although the first branch of methods improves the splitting process per dimension, as long as the splits are axis-parallel, the iForest algorithm still creates artificial outlier regions, and shows low performances in detecting hidden outlier subspaces. 

Another characteristic of the AIDA algorithm is that, so far in the literature, the isolation score has been computed purely using a simulative approach. Conversely, here we prove that the outlier score function used by AIDA admits an analytical closed-form expression, which simplifies computations and provides new insights into the isolation metric. These analytical formulas can be used---for instance---to find deeper connections between isolation and neighbour-based methods, or to analyze the theoretical properties of the iForest algorithm in simple scenarios.

Since the distance measure loses contrast in very high dimensions, we use an ensemble of random subspaces in order to alleviate this problem (we refer to \cite{aggarwal2017,keller2012,lazaveric2005} for several subspace sampling alternatives). Several outlier scores and splitting distributions are also tested with the purpose of finding a good measure of ``outlierness" in high dimensions, reducing the curse of dimensionality. We also employ random subsampling to bring the computational complexity from quadratic to linear in the number of samples \cite{pang2015}. Due to the properties of the ensemble, the algorithm can benefit from parallelization to further reduce the computational burden.

Another fundamental aspect in anomaly detection is the ability to explain why a certain point was labelled as an outlier \cite{dang2014}. In many applications, practitioners are faced with large data sets containing hundreds or even thousands of features. An anomaly detection algorithm that only informs whether a point is an outlier or an inlier is much less helpful than an algorithm that also returns the most important features defining the outliers. This information can be used to focus on the outliers that seem more interesting in a particular application, greatly reducing the time analysts need to spend studying outliers. For this reason, we propose an explanation method that combines the AIDA method and the Simulated Annealing (SA) algorithm (e.g., \cite{aarts1985,kirkpatrick1983}), the \textit{Tempered Isolation-based eXplanation} (TIX) algorithm. The TIX algorithm satisfies the four desirable properties for anomaly explanation introduced in \cite{gupta2019}, namely: 1) it has quantifiable explanations, 2) it is not computationally expensive, 3) it is visually interpretable and 4) scalable. Moreover, it also takes into account the interactions among features, and it is able to find outliers hidden in multidimensional subspaces \cite{keller2012}.


The paper is organized as follows. In \Cref{sec:method}, we introduce the AIDA algorithm as well as the analytical formulas for isolation. We also show with a simple example the type of outliers that AIDA detects when compared to iForest and LOF. The TIX method is described in \Cref{sec:explain}. Numerical results concerning the performances of AIDA and TIX are given in \Cref{sec:results}. \Cref{sec:conclusions} concludes the paper.

\section{Methodology}\label{sec:method}

We first introduce the AIDA algorithm for numerical features in \Cref{sec:gen_set}, with the analytical formulas for isolation in \Cref{sec:analy_iso}. Then, we present a possible extension of AIDA to categorical features in \Cref{sec:aida_cat}. Finally \Cref{sec:aida_example} illustrates the type of outliers that AIDA detects with a simple example.

\subsection{General setting}\label{sec:gen_set}

Let $\bm{X}_n$ be a data set of size $n$ and dimensionality $d$, such that $X_i \in \mathbb{R}^d$, for $i=1,...,n$, and let $l_p(\cdot,\cdot)$ be a weighted distance function defined as
\begin{equation}\label{eq:distance}
    l_p(X_i,X_j) = \left(\sum_{l=1}^d \omega_l \, |X_{i,l}-X_{j,l}|^p\right)^{1/p},
\end{equation}
where $p \in \mathbb{R}^+$ and $\omega_l \in \mathbb{R}^+$ for $l = 1,...,d$ are the weights given to each feature\footnote{We focus on $\mathcal{L}_p$ norms only, but other notions of distance, e.g., cosine distances, can also be applied.}.

Moreover, let $N,\psi_{min},\psi_{max} \in \mathbb{N}^+$ be the number of subsamples, the mininum subsampling size and the maximum subsampling size, respectively. Then, the AIDA algorithm works as follows: first, we create $N$ random subsamples $\bm{Y}_{\psi_j}$, for $j = 1,...,N$, from $\bm{X}_n$ without replacement with sizes $\psi_j$ ranging randomly between $\psi_{min}$ and $\psi_{max}$. This is the training stage, which is simply storing the subsamples of the training set for future use. The average and worst case memory requirements are $\mathcal{O}(Nd\psi_m)$ and $\mathcal{O}(Nd\psi_{max})$, respectively, where $\psi_m = (\psi_{min}+\psi_{max})/2$.

Next, for each point $X_i$ in the test set (for simplicity, we assume the test set is $\bm{X}_n$) we compute its distance to every observation in a given subsample $\bm{Y}_{\psi_j}$---for $j = 1,...,N$---using \Cref{eq:distance}, and sort them in increasing order. We also include the zero point into the distances, which corresponds to the distance of $X_i$ to itself. Thus, the minimum distance is always zero, which we denote as the \textit{left-fringe point}, since it is the left-most point in the sorted distances. We call this projection the \textit{distance profile} (DP) of a point $X_i$ with respect to a subsample $\bm{Y}_{\psi_j}$, denoted DP$(X_i,\bm{Y}_{\psi_j})$, for $j = 1,...,N$.

Once the distances have been sorted, we apply the iForest algorithm to this new data set until the  left-fringe point has been isolated.  These steps give us an outlier score per subsample, and the final score of $X_i$ is obtained by aggregating these results, usually with the average or the maximum functions \cite{aggarwal2015}.

The idea of the AIDA algorithm is illustrated in \Cref{fig:example_2D}. The top left plot shows the complete data set, which consists of 1000 observations with two features, while the top right plot shows a random subsample of size 50, together with two test points marked with a red triangle (A) and a red circle (B). The lower plots present the DPs of point A (left) and point B (right), where the left-fringe point is marked with a red cross to emphasize that this is the point we want to isolate. Clearly, the left-fringe point is easier to isolate in the DP of point A than in the DP of point B, hence point A will receive a higher outlier score. This is expected by looking at \Cref{fig:example_2D_data}, since point A is in an area of much lower density. An outlier is, therefore, a point that is easy to isolate in the 1D projection given by its DP.

\begin{figure}[ht]
    \centering
    \begin{subfigure}{0.4\textwidth}
    \centering
    \includegraphics[width=\textwidth]{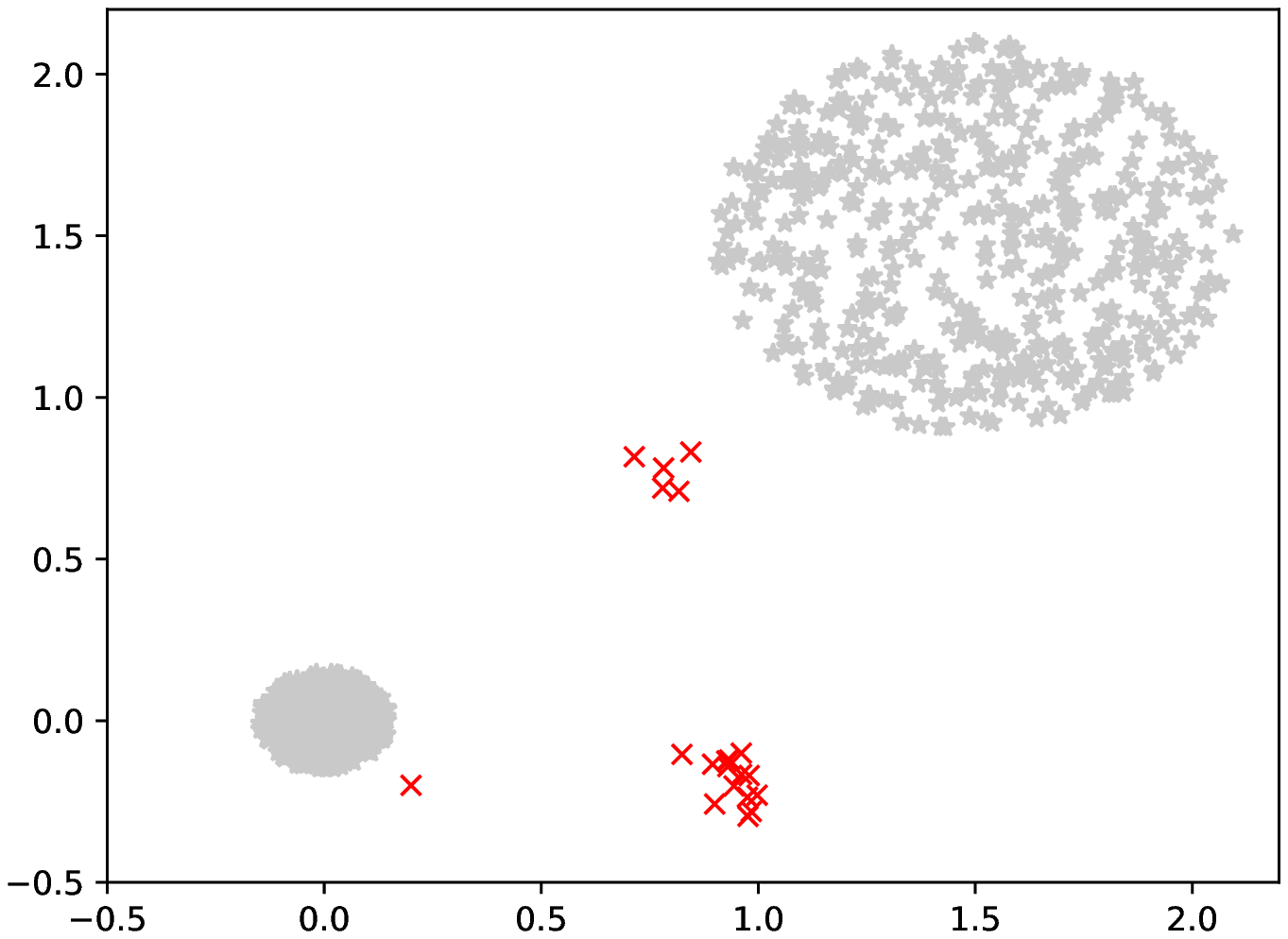}
    \caption{Whole data set.}
    \label{fig:example_2D_data}
    \end{subfigure}
    \begin{subfigure}{0.4\textwidth}
    \centering
    \includegraphics[width=\textwidth]{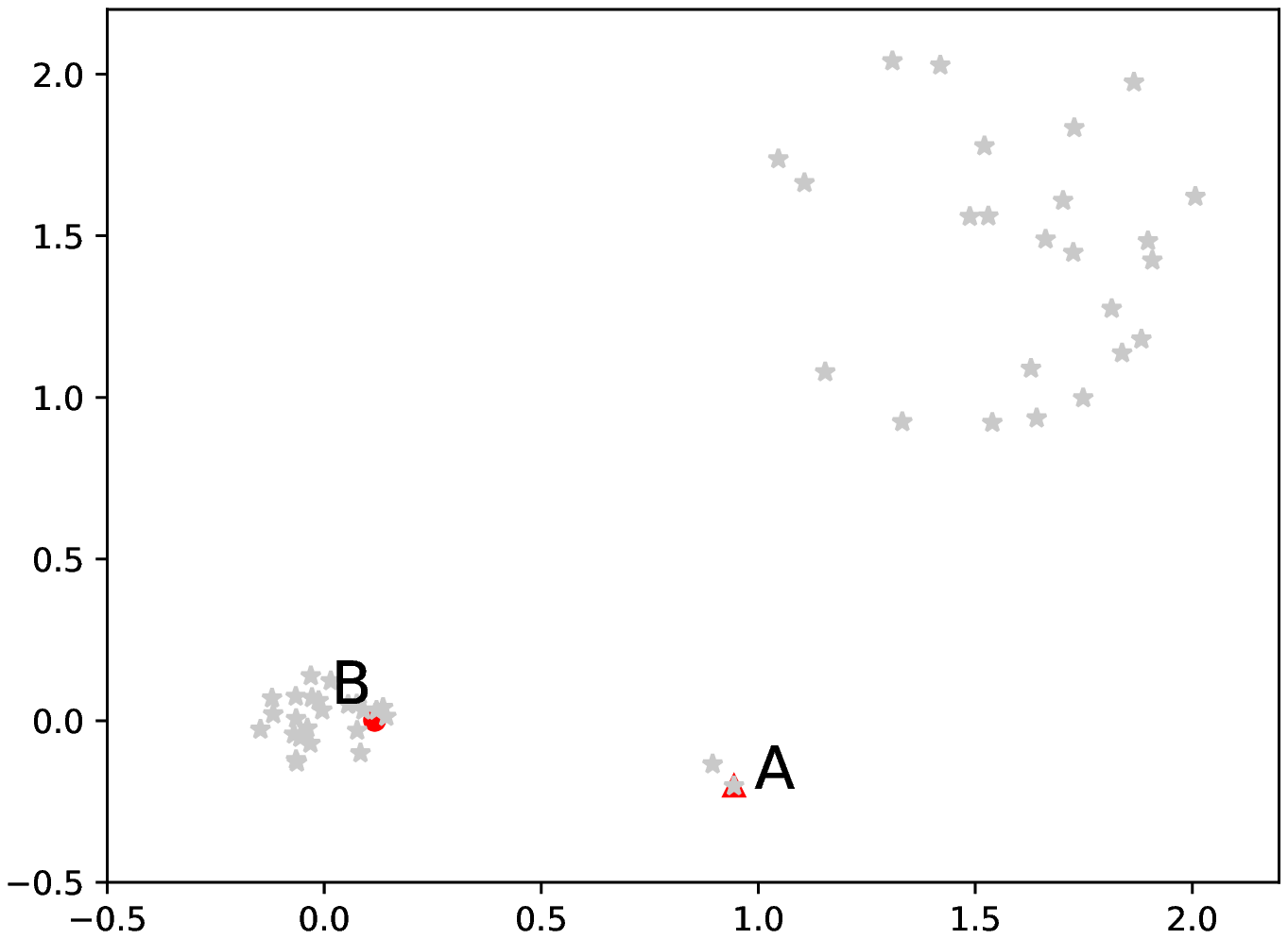}
    \caption{Random subsample.}
    \end{subfigure}
    \medskip
    \begin{subfigure}{0.4\textwidth}
    \centering
    \includegraphics[width=\textwidth]{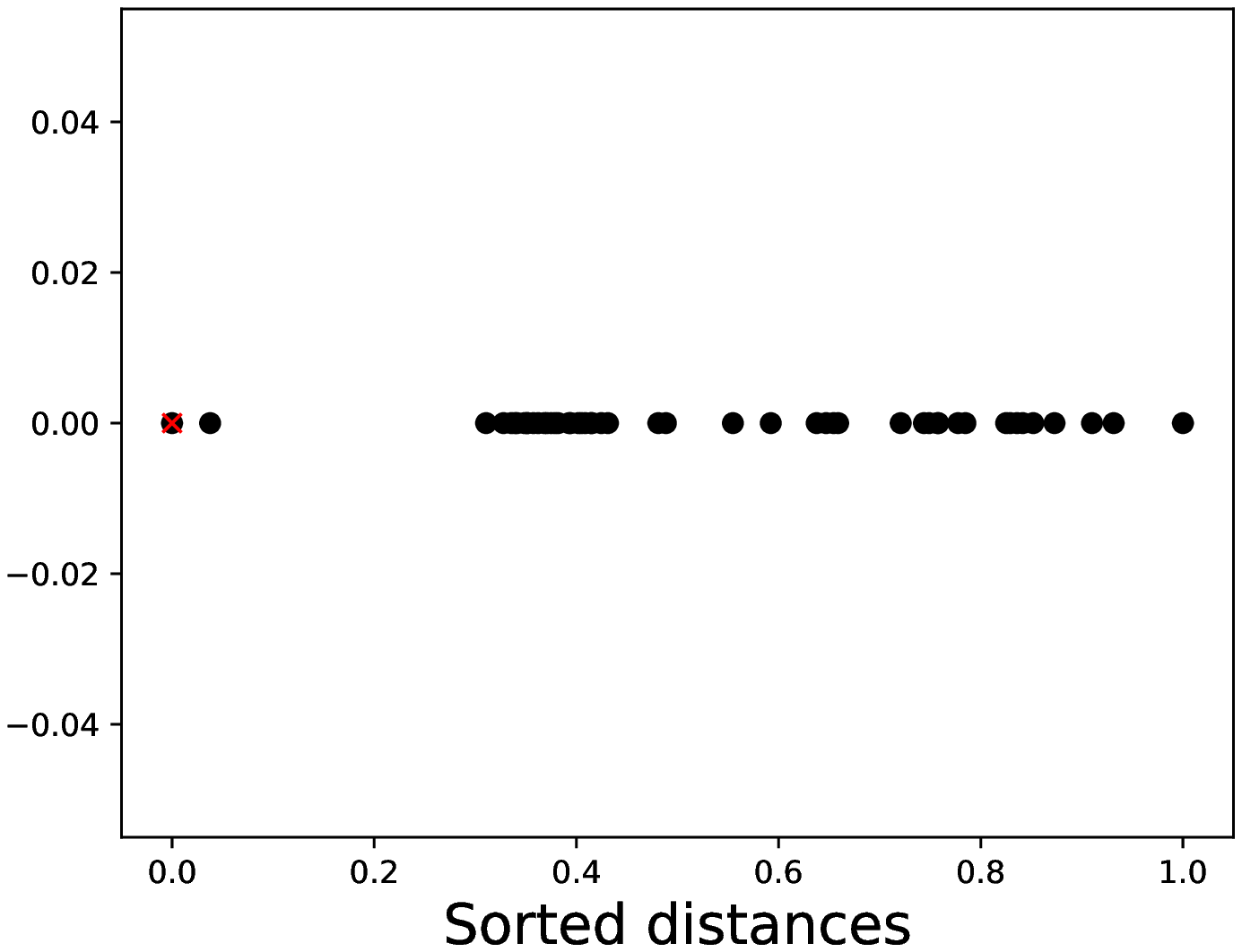}
    \caption{Distance profile of point A.}
    \label{fig:example_2D_dist1}
    \end{subfigure}
    \begin{subfigure}{0.4\textwidth}
    \centering
    \includegraphics[width=\textwidth]{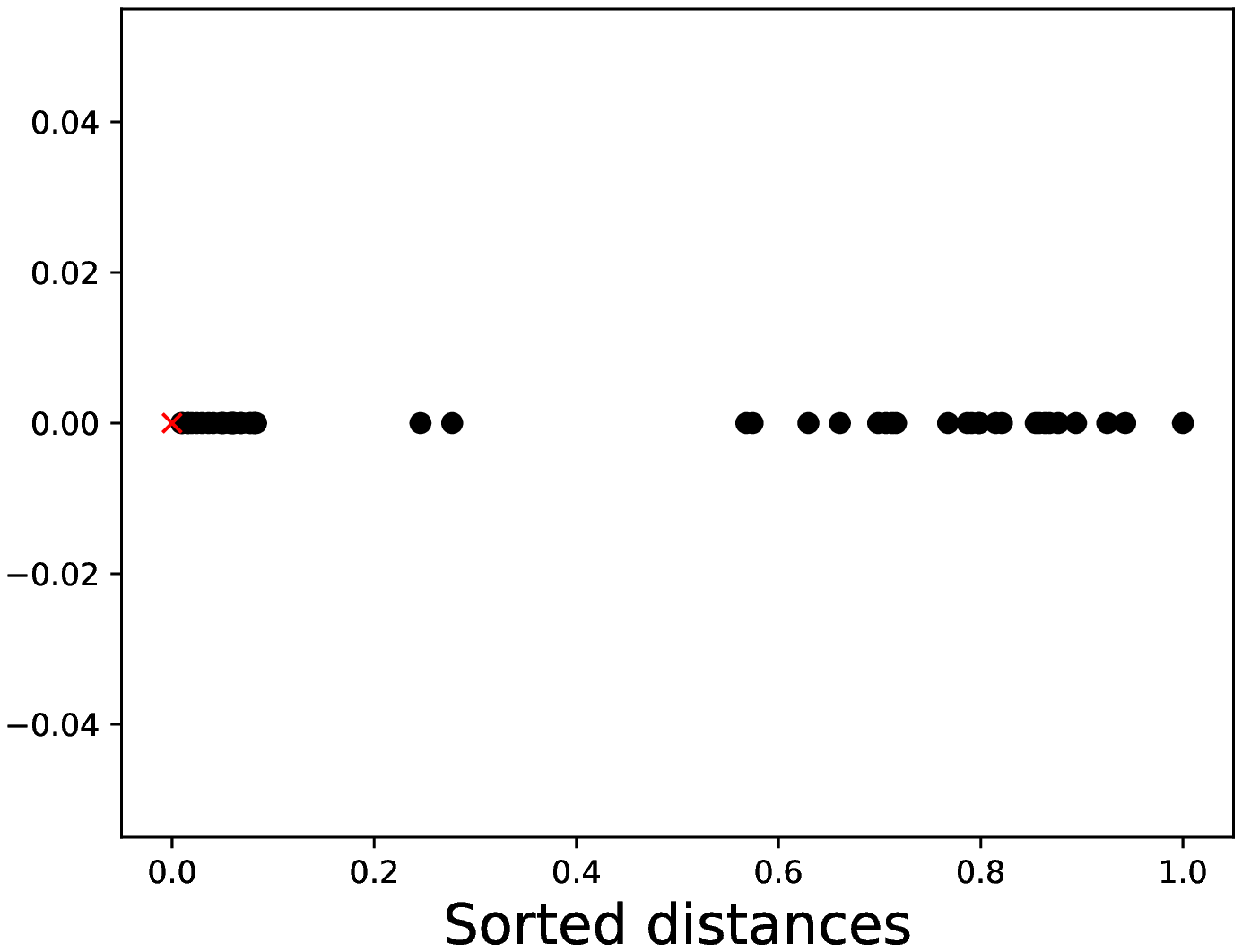}
    \caption{Distance profile of point B.}
    \label{fig:example_2D_dist2}    
    \end{subfigure}    
    
    \caption{Comparison between the DPs of an outlier (A) and an inlier (B). The inlier (B) is marked as a red circle on the top right figure, and the outlier (A) is marked with a red triangle.}
    \label{fig:example_2D}
\end{figure}

\subsection{Analytical Isolation}\label{sec:analy_iso}

If we follow the standard iForest methodology, once the DP has been computed, we would randomly split the data using simulations until the left-fringe point has been isolated. However, since the DP is a 1D projection of the full feature space, we can benefit from analytical formulas to compute the isolation score, hence reducing the computational cost and the variance of the results. This is proved in the following proposition.

\begin{prop}\label{prop:mgf}
Let $\bm{Z}_n$ be a sorted vector of real numbers such that $Z_i \in \mathbb{R}, i=1,...,n$, and $Z_i \leq Z_j, i \leq j$. Denote by $h$ the number of splits that it takes to isolate $Z_1$, and by $g(Z_i,Z_{i+1}|\bm{Z}_n)$ the probability of a random split occurring on the interval $[Z_i,Z_{i+1})$ given $\bm{Z}_n$. Assume that\footnote{In the original iForest algorithm, the split probabilities are directly proportional to the length of the interval, so that $g(Z_i,Z_{i+1}) = Z_{i+1}-Z_i$. Here we consider a more general formulation \cite{tokovarov2022}.}
\begin{equation}
    g(Z_i,Z_{i+1}|\bm{Z}_n) = \frac{g(Z_i,Z_{i+1})}{G(Z_1,Z_n)},
\end{equation}
with $G(Z_1,Z_i) = \sum_{j=1}^{i-1} g(Z_j,Z_{j+1})$, then the moment generating function (mgf) of $h$ is given by
\begin{equation}\label{eq:Eeu_general}
    \mathbb{E}[e^{u h}|\bm{Z}_n] = \prod_{i=1}^{n-1} \frac{e^u g(Z_i,Z_{i+1}) + G(Z_1,Z_i)}{G(Z_1,Z_{i+1})}.
\end{equation}
\end{prop}
\begin{proof}
Since $h$ is bounded between $1$ and $n-1$, it is clear that, for any function $f:\mathbb{R}\rightarrow \mathbb{R}$,
\begin{equation}\label{eq:Efh_init}
    \mathbb{E}[f(h)|\bm{Z}_n] = \sum_{i=1}^{n-1} f(i)\mathbb{P}[h=i|\bm{Z}_n].
\end{equation}
Furthermore, due to the recursive nature of the random splits,
\begin{equation}\label{eq:prob_recursive}
    \mathbb{P}[h=i|\bm{Z}_n] = \sum_{j=i}^{n-1} g(Z_j,Z_{j+1}|\bm{Z}_n)\mathbb{P}[h=i-1|\bm{Z}_j],
\end{equation}
where $\bm{Z}_j$ contains the first $j$ elements of $\bm{Z}_n$.

Plugging \Cref{eq:prob_recursive} into \Cref{eq:Efh_init} and swapping the order of the summations yields
\begin{equation}
    \mathbb{E}[f(h)|\bm{Z}_n] = f(1)g(Z_1,Z_2|\bm{Z}_n) + \sum_{j=2}^{n-1} g(Z_j,Z_{j+1}|\bm{Z}_n) \sum_{i=2}^{j} f(i) \mathbb{P}[h=i-1|\bm{Z}_j],
\end{equation}
which can be further simplified as
\begin{equation}\label{eq:Efh_mid}
    \mathbb{E}[f(h)|\bm{Z}_n] = f(1)g(Z_1,Z_2|\bm{Z}_n) + \sum_{j=2}^{n-1} g(Z_j,Z_{j+1}|\bm{Z}_n) \mathbb{E}[f(h+1)|\bm{Z}_j].
\end{equation}
\Cref{eq:Efh_mid} can then be written as a recursion, i.e., 
\begin{equation}\label{eq:Efh_recursion}
    \mathbb{E}[f(h)|\bm{Z}_n] = \frac{g(Z_{n-1},Z_n)}{G(Z_1,Z_n)}\mathbb{E}[f(h+1)|\bm{Z}_{n-1}] + \frac{G(Z_1,Z_{n-1})}{G(Z_1,Z_n)}\mathbb{E}[f(h)|\bm{Z}_{n-1}],
\end{equation}
where the first term on the right-hand side of \Cref{eq:Efh_recursion} corresponds to the last term of the summation appearing in \Cref{eq:Efh_mid}, and the normalization factor $G(Z_1,Z_{n-1})/G(Z_1,Z_n)$ is necessary to move from probabilities conditioned on $\bm{Z}_n$ to probabilities conditioned on $\bm{Z}_{n-1}$.

Finally, choosing $f(h) = e^{u h}$ gives
\begin{equation}
    \mathbb{E}[e^{u h}|\bm{Z}_n] = \frac{e^u g(Z_{n-1},Z_n) + G(Z_1,Z_{n-1})}{G(Z_1,Z_n)}\mathbb{E}[e^{u h}|\bm{Z}_{n-1}],
\end{equation}
and, by following the recursion, one obtains the desired result of \Cref{eq:Eeu_general}.
\end{proof}

Given the mgf in \Cref{eq:Eeu_general}, it is easy to compute other quantities---such as the expectation and the variance of $h$---through its derivatives with respect to $h$.

\begin{corollary}
Let $\bm{Z}_n$, $h$ and $g(Z_i,Z_{i+1}|\bm{Z}_n)$ be defined as in \Cref{prop:mgf}. Then
\begin{equation}\label{eq:mean_general}
    \mathbb{E}[h|\bm{Z}_n] = 1 + \sum_{i=2}^{n-1}\frac{g(Z_i,Z_{i+1})}{G(Z_1,Z_{i+1})},
\end{equation}

\begin{equation}\label{eq:var_general}
    \mathbb{V}[h|\bm{Z}_n] = \sum_{i=2}^{n-1}\frac{g(Z_i,Z_{i+1})}{G(Z_1,Z_{i+1})}\left(1-\frac{g(Z_i,Z_{i+1})}{G(Z_1,Z_{i+1})}\right).
\end{equation}
\end{corollary}
\begin{proof}
The result is immediate after a straightforward computation of the derivatives of $\log \mathbb{E}[e^{u h}|\bm{Z}_n]$.
\end{proof}

While \Cref{prop:mgf} provides a general framework for isolation in one dimension, from now on we set $g(Z_i,Z_{i+1}) = (Z_{i+1}-Z_i)^{\alpha}$ as in \cite{tokovarov2022}. Applying this choice to \Cref{eq:mean_general,eq:var_general} yields the two scoring functions that we will consider in this paper.

\begin{equation}\label{eq:mean_expo}
    \mathbb{E}[h|\bm{Z}_n] = 1 + \sum_{i=2}^{n-1}\frac{(Z_{i+1}-Z_i)^{\alpha}}{\sum_{j=1}^i(Z_{j+1}-Z_j)^{\alpha}},
\end{equation}

\begin{equation}\label{eq:var_expo}
    \mathbb{V}[h|\bm{Z}_n] = \sum_{i=2}^{n-1}\frac{(Z_{i+1}-Z_i)^{\alpha}}{\sum_{j=1}^i(Z_{j+1}-Z_j)^{\alpha}}\left(1-\frac{(Z_{i+1}-Z_i)^{\alpha}}{\sum_{j=1}^i(Z_{j+1}-Z_j)^{\alpha}}\right).
\end{equation}

With this formulation, the split probabilities used in iForest are a particular case of \Cref{eq:mean_expo,eq:var_expo} with $\alpha = 1$. The scoring functions can thus be further simplified as

\begin{equation}\label{eq:mean_uni}
    \mathbb{E}[h|\bm{Z}_n] = 1 + \sum_{i=2}^{n-1}\frac{Z_{i+1}-Z_i}{Z_{i+1}-Z_1},
\end{equation}

\begin{equation}\label{eq:var_uni}
    \mathbb{V}[h|\bm{Z}_n] = \sum_{i=2}^{n-1}\frac{Z_{i+1}-Z_i}{Z_{i+1}-Z_1}\left(1-\frac{Z_{i+1}-Z_i}{Z_{i+1}-Z_1}\right).
\end{equation}

The intuition behind using $\mathbb{E}[h]$ as a score function was already explained in the original iForest paper \cite{liu2012}: an outlier should be isolated in just a few splits and, therefore, $\mathbb{E}[h]$ should be smaller for outliers than for inliers. The reason for using the variance as an alternative, as we propose here, is similar in nature. An inlier should be surrounded by other inliers, and, therefore, it takes---on average---many splits to isolate it. However, albeit less probable, an inlier can also be isolated in just a few splits, giving rise to a large range of variation for $h$. Hence, it is assumed that an inlier has a higher $\mathbb{V}[h]$ than an outlier. In the extreme case where a point is always isolated in a single split, we have $\mathbb{V}[h] = 0$.

In the context of the AIDA algorithm, $\bm{Z}_n$ corresponds to the DP of a given point, so that $Z_1 = 0$, and the other $Z_i$'s are the sorted distances. For example, assuming no subsampling ($\bm{Y}_{\psi_j} = \bm{X}_n$), the outlier score of a point $X$ using \Cref{eq:var_expo} as the score function would be
\begin{equation}
    score(X) = -\mathbb{V}[h|DP(X,\bm{X}_n)],
\end{equation}
where the minus sign is added so that the score of outliers is higher than the score of inliers.

Hence, the AIDA algorithm can be decomposed into two main steps. First, the DP of each point $X_i$ in $\bm{X}_n$ is computed with respect to each subsample $\bm{Y}_{\psi_j}$, for $j = 1,...,N$. Then, an outlier score is obtained by applying either \Cref{eq:mean_expo} or \Cref{eq:var_expo} to the DPs and aggregating the results among all subsamples. The average computational complexity of this procedure is $\mathcal{O}(n N \psi_m  (d+\log(\psi_m)+1))$. Therefore, it is linear in both the number of features and the number of observations. Additionally, if $d>\log(\psi_m)$, the pair-wise distances are the most expensive part of the algorithm. Otherwise, it is the sorting function that takes most of the computational time. The score function does not consume much time in comparison.

\begin{remark}
\Cref{eq:mean_expo,eq:var_expo} may diverge if there are consecutive zero values, therefore these cases must be treated separately. In practice, the challenge are the points equal to $Z_1$, since, for instance, in the case $Z_1 \neq Z_2 = Z_3$ the denominator is still larger than zero. Moreover, repeated values equal to $Z_1$ are, by definition, impossible to isolate, thus we recommend the maximum penalization to the outlier score for each repeated value. In particular, we suggest to add $+1$ and $+0.25$ to \Cref{eq:mean_expo,eq:var_expo} for each repeated value, respectively. The reason for choosing these values is that they represent the maximum possible increments per observation in \Cref{eq:mean_expo,eq:var_expo}. 
\end{remark}

\subsection{Categorical data and subspace search}\label{sec:aida_cat}

The AIDA methodology can be coupled with subspace search methods such as feature bagging, rotated bagging, and others \cite{lazaveric2005,aggarwal2015}. Doing so, partially diminishes the problem of loss of contrast that distance-based methods face in high dimensions. The drawback is that we would then introduce further randomness into the algorithm, and many subspaces must be explored to obtain meaningful results. Nevertheless, subspace search methods have shown to produce satisfactory results when compared to a full space search \cite{lazaveric2005,keller2012}. Additionally, if we couple each random subspace with a random subsample, the computational cost can be further reduced, since the time complexity of the distance calculation is linear in the number of features.

Another important generalization is the inclusion of categorical data, in particular nominal data. This is because there is no clear nor unambiguous relationship between the distinct values of a nominal feature\footnote{In categorical variables, especially when non-ordinal, the way in which categories are defined has a substantial impact on their interpretability, and thus on measures meant to quantify similarity, dispersion, etc. \cite{Agresti2013}.}, and thus the concept of distance cannot be directly applied. Hence, we consider instead the concept of \textit{similarity} between two different categories of a nominal feature, and then transform this similarity into a distance metric. 

In \cite{boriah2008}, the following relationship between distance and similarity was considered:
\begin{equation}\label{eq:sim_dist_old}
    S(X,Y) = \frac{1}{1+dist(X,Y)},
\end{equation}
so that points with similarity one have zero distance.

Here, we consider an analogous relation, namely
\begin{equation}\label{eq:sim_dist}
    S(X,Y) = e^{-dist(X,Y)}.
\end{equation}

\Cref{eq:sim_dist_old,eq:sim_dist} coincide in the extremes of similarity one and zero, but the rate of convergence towards zero in terms of the distance is much faster in \Cref{eq:sim_dist} than in \Cref{eq:sim_dist_old}. 

The next step is to choose a particular similarity function. Here we have chosen to work with a slight modification of the \textit{Goodall3} similarity function \cite{boriah2008}, but several alternatives are possible \cite{pang2015}. 

Let $\bm{X}_n^{nom}$ be a data set of size $n$ consisting of $d_{nom}$ nominal features, such that $X_i^{nom} \in \mathbb{N}^{d_{nom}}$, for $i = 1,...,n$. Moreover, let $f_k(x)$ be the number of times that class $x$ appears in the $k$-th feature of $\bm{X}_n^{nom}$. Then, the similarity between two points with classes $x$ and $y$ in a given nominal feature $k$ is defined as
\begin{equation}\label{eq:sim_feature}
    S_k(x,y) = \left\{ \begin{array}{ll}
                1, & x = y,\\
                \\
                1-p_k^2(y), & x \neq y,
                \end{array} \right.
\end{equation}
where
\begin{equation}\label{eq:prob_nom}
  p_k^2(x) = \frac{f_k(x)(f_k(x)-1)}{(n+1)n}.
\end{equation}

Finally, combining \Cref{eq:sim_dist,eq:sim_feature} we obtain the distance between two samples consisting of $d_{nom}$ nominal features:
\begin{equation}\label{eq:dist_sim}
    dist^{nom}(X_i^{nom},X_j^{nom}) = -\sum_{l=1}^{d_{nom}} \omega_l^{nom} \log(S_l(X_{i,l}^{nom},X_{j,l}^{nom})).
\end{equation}

Notice that the denominator in \Cref{eq:prob_nom} contains the term $(n+1)$, instead of the original $(n-1)$ in \cite{boriah2008}. This is to avoid a similarity of exactly zero, which would cause the distance to diverge and yield unstable results in the anomaly detection algorithm. Thus, the distance defined in \Cref{eq:dist_sim} is capped to a maximum of $\log((n+1)/2)$ per feature.

\begin{remark}
\Cref{eq:dist_sim} implies a different similarity aggregation than the one commonly used in the literature \cite{boriah2008}. In particular, the total similarity is usually defined as the weighted average of the similarities per feature, while here we have defined it as the weighted product instead. The purpose of this change is to magnify the effect of features where the classes do not match. Consider an example where $d_{nom}$ is very large, and $X_i^{nom}$ and $X_j^{nom}$ match in every feature except one, where $X_i^{nom}$ has a unique class and could, therefore, be labelled as an outlier. Using the average similarity would yield a total similarity close to one---or a distance close to zero---possibly resulting in the wrong classification of $X_i^{nom}$ as an inlier. The weighted product helps solving this problem, and detects nominal features where outliers are different from the majority of the data.
\end{remark}

Having defined a distance measure for the nominal features, the general scenario with mixed-attribute data can be tackled in the following manner: consider a data set $\bm{X}_n = \{\bm{X}_n^{num},\bm{X}_n^{nom}\}$ consisting of $n$ samples with $d = d_{num}+d_{nom}$ features, where the first $d_{num}$ are numerical\footnote{This includes categorical ordinal data.} and the last $d_{nom}$ are nominal. Then the total distance between two observations in $\bm{X}_n$ is given by
\begin{equation}\label{eq:total_distance}
    dist(X_i,X_j) = l_p(X_i^{num},X_j^{num})+dist^{nom}(X_i^{nom},X_j^{nom}),
\end{equation}
where $l_p(\cdot,\cdot)$ and $dist^{nom}(\cdot,\cdot)$ are defined by \Cref{eq:distance,eq:dist_sim}, respectively. The rest of the algorithm is the same as in the case with only numerical features. That is, we compute the DPs of each observation based on \Cref{eq:total_distance} and compute the outlier score using either \Cref{eq:mean_expo} or \Cref{eq:var_expo}. 

Pseudocodes illustrating the training and test phases of the AIDA method are presented in \Cref{alg:aida_training} and \Cref{alg:aida_testing}, respectively. Notice that we have introduced a normalization step in line 12 of \Cref{alg:aida_testing}. This is particularly relevant when each subsample has assigned a different feature subspace, since in that case the unnormalized outlier scores may not be comparable.

\begin{algorithm}
\caption{AIDA: training phase.}\label{alg:aida_training}
\begin{algorithmic}[1]
\State Load $\bm{X}_n^{num}$ and $\bm{X}_n^{nom}$.
\State Set $N$, $\psi_{min}$, $\psi_{max}$.
\For{$j = 1,...,N$}
\State Set $\psi_j \sim U(\psi_{min},\psi_{max})$.
\State Set $\bm{Y}_{\psi_j}$ by drawing $\psi_j$ samples without replacement from $\bm{X}_n$.
\State Compute and store the frequencies of each class of the nominal features using \Cref{eq:prob_nom}.
\EndFor
\end{algorithmic}
\end{algorithm}

\begin{algorithm}
\caption{AIDA: testing phase.}\label{alg:aida_testing}
\begin{algorithmic}[1]
\State Load $\bm{X}_n^{num}$ and $\bm{X}_n^{nom}$.
\State Choose a score function from \Cref{eq:mean_expo,eq:var_expo}.
\State Set $\alpha$, $\bm{\omega}^{num}$ and $\bm{\omega}^{nom}$.
\For{$i = 1,...,n$}
    \For{$j = 1,...,N$}
       \State Compute the distance of $X_i$ to each point in $\bm{Y}_{\psi_j}$ and to itself using \Cref{eq:distance,eq:dist_sim,eq:total_distance}.
       \State Sort the distances from minor to major.
       \State Compute the outlier score with the chosen outlier function.
    \EndFor       
\EndFor
\For{$i = 1,...,N$}
    \State Transform the outlier scores to Z-scores.
\EndFor
\State Aggregate the scores obtained with each subsample.
\end{algorithmic}
\end{algorithm}

\begin{remark}
The weights $\omega_l$ in \Cref{eq:distance,eq:dist_sim} can be used to emphasize or diminish the contribution of specific features. Specifically, a large $\omega_l$ gives more importance to the $l$-th feature in the distance metric, hence points that are anomalous in that feature will be found more easily. This property is relevant when there is some prior knowledge about the features that cause the outliers. These weights can also be interpreted as a generalization of several subspace search methods \cite{aggarwal2017}. For example, we can implement the feature bagging \cite{lazaveric2005} algorithm by setting $\omega_l = 1$ on the randomly chosen features, and $\omega_l = 0$ on the rest. In case of no prior knowledge about the outliers, we recommend setting $\omega_l = 1$, for $l=1,...,d$, if no subspace search methods are used.
\end{remark}

\subsection{Illustrative example}\label{sec:aida_example}

We return to the data set of \Cref{fig:example_2D} and compare the type of outliers AIDA detects with iForest and LOF. This example also allows us to compare the two scoring functions for different values of $\alpha$ in \Cref{eq:mean_expo,eq:var_expo}. For the distance-based methods, we use the Manhattan distance with equal weights, so $p = 1$ and $\omega_l = 1$ in \Cref{eq:distance}, for $l=1,...,d$. We use the AIDA algorithm without subsampling, so that there is no source of randomness and the main difference with LOF is simply the use of the isolation score instead of the local density. We test the expectation and variance score functions of \Cref{eq:mean_expo,eq:var_expo}, respectively, with two values of $\alpha$, mainly $\alpha = 1$ and $\alpha = 2$, for a total of four AIDA configurations. We denote, for example, the AIDA algorithm with variance score function and $\alpha = 1$ as AIDA (V1), and the other variations analogously. We have chosen the number of neighbours in LOF to be $k = 20$, which seems reasonable given the size of the data set. Regarding iForest, we set the number of trees to $1000$ and we do not use subsampling.

The results can be seen in \Cref{fig:example_2D_outliers}, where we plot the 60 most anomalous points detected by each method. It is of interest to observe that each method identifies different parts of the data set as outliers, apart from the most obvious ones. Concretely, both iForest and LOF assign outliers to the rims of the inlier clusters. Given that the local densities of the inlier clusters are very similar, LOF outliers can be seen in both clusters, while iForest only detects anomalies in the larger cluster.

In contrast, AIDA is able to find sparse areas inside the inlier clusters. For $\alpha = 1$, the expectation and variance scores behave similarly, detecting outliers only in the large cluster. However, the expectation score fails to detect a couple of the actual outliers, while the variance score detects all of them. Setting $\alpha = 2$ gives more importance to the small inlier cluster, and so both score functions detect outliers in the spare areas of that cluster as well. On the other hand, the detection of the actual outliers has worsened, especially in the expectation score. Increasing $\alpha$ even further appears to be detrimental in this example, as it magnifies the noise in the data.

Nonetheless, these results hint that the variance function may be a better score function than the expectation, and that different values of $\alpha$ may be used to detect various types of outliers. This will be further explored in \Cref{sec:results}.

\begin{figure}[ht]
    \centering
    \begin{subfigure}{0.4\textwidth}
    \centering
    \includegraphics[width=\textwidth]{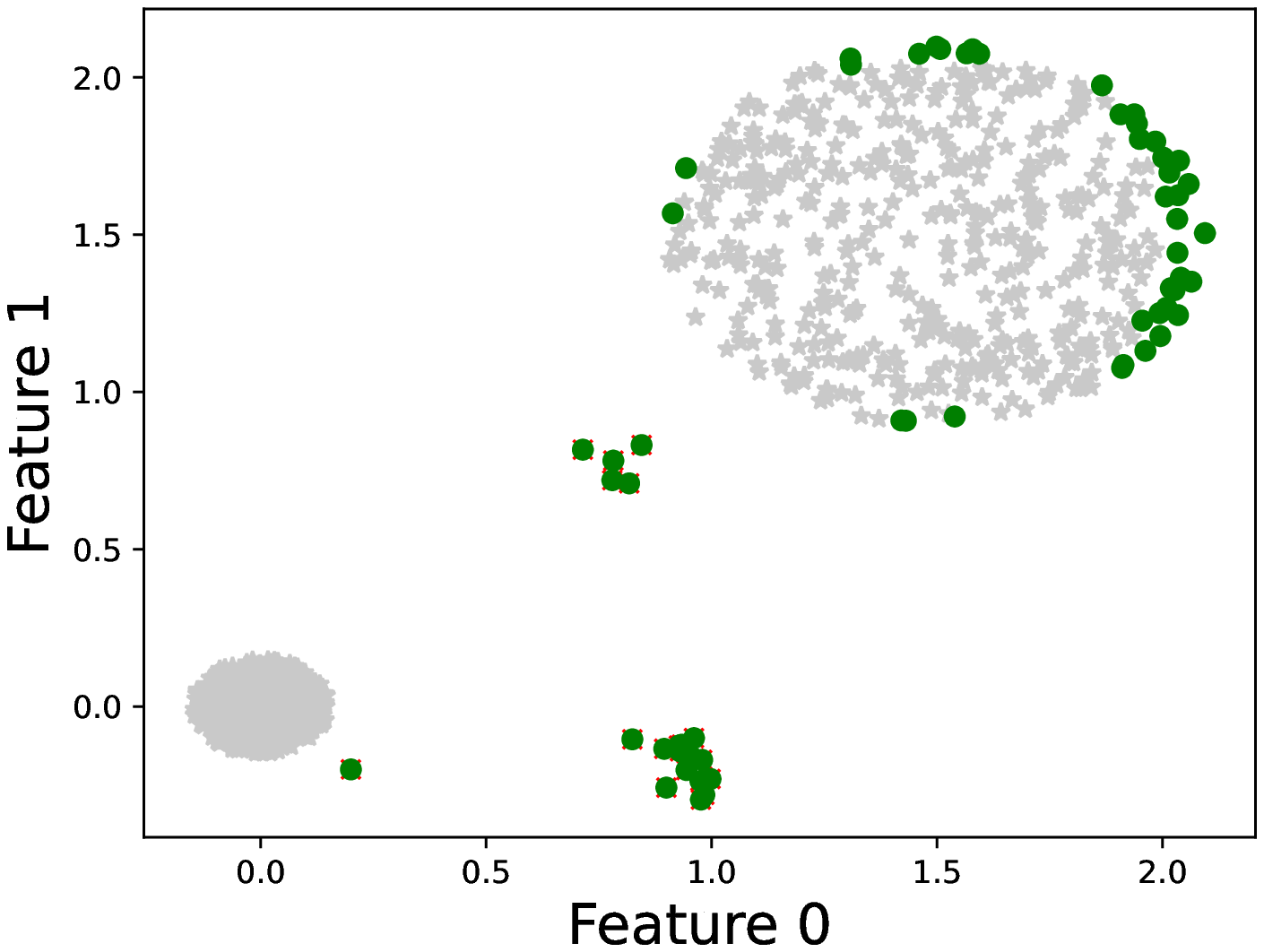}
    \caption{Isolation Forest.}
    \label{fig:example_2D_IF}
    \end{subfigure}
    \begin{subfigure}{0.4\textwidth}
    \centering
    \includegraphics[width=\textwidth]{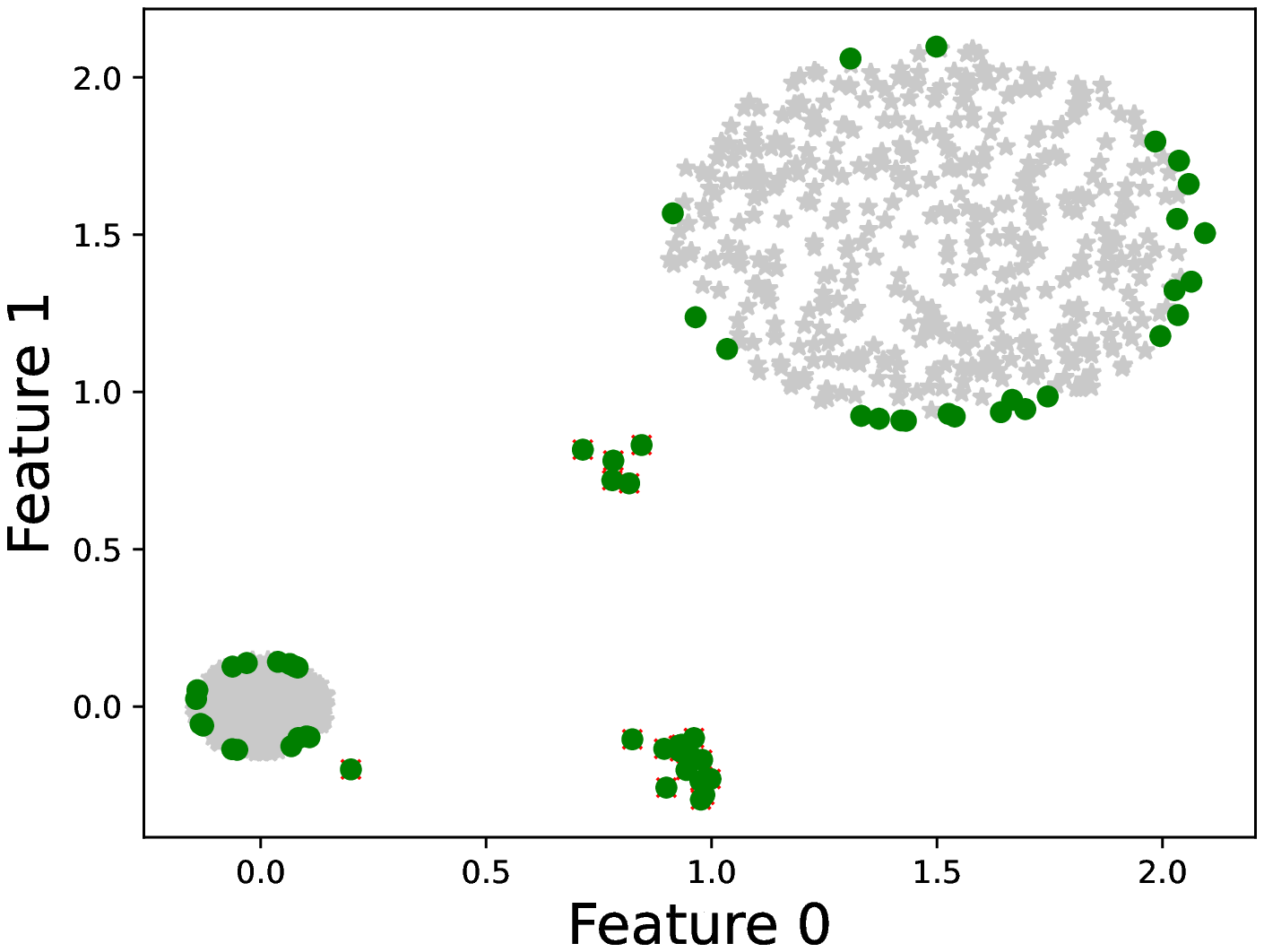}
    \caption{Local Outlier Factor.}
    \end{subfigure}
    \medskip
    \begin{subfigure}{0.4\textwidth}
    \centering
    \includegraphics[width=\textwidth]{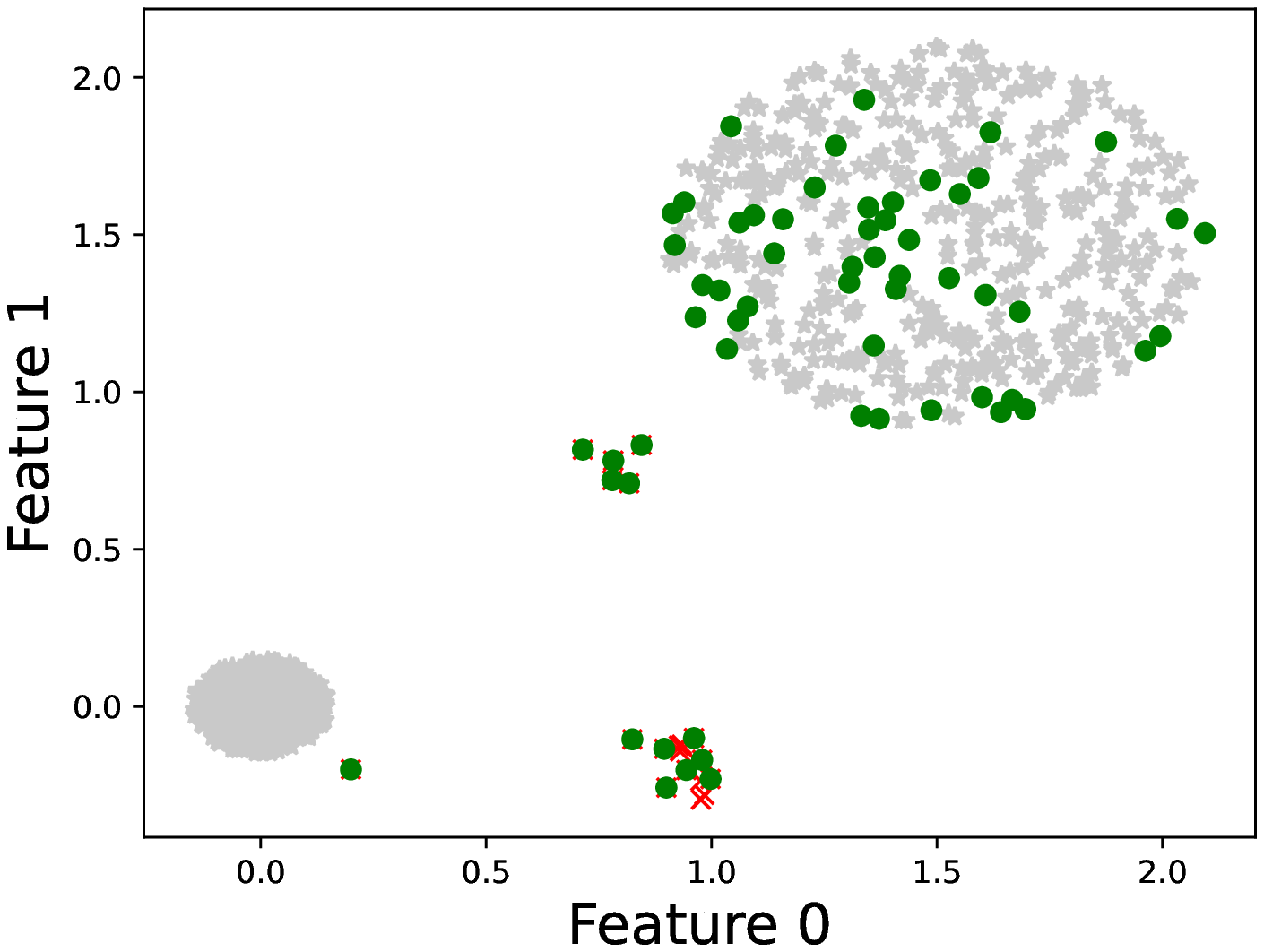}
    \caption{AIDA (E1).}
    \label{fig:example_2D_expectation1}
    \end{subfigure}
    \begin{subfigure}{0.4\textwidth}
    \centering
    \includegraphics[width=\textwidth]{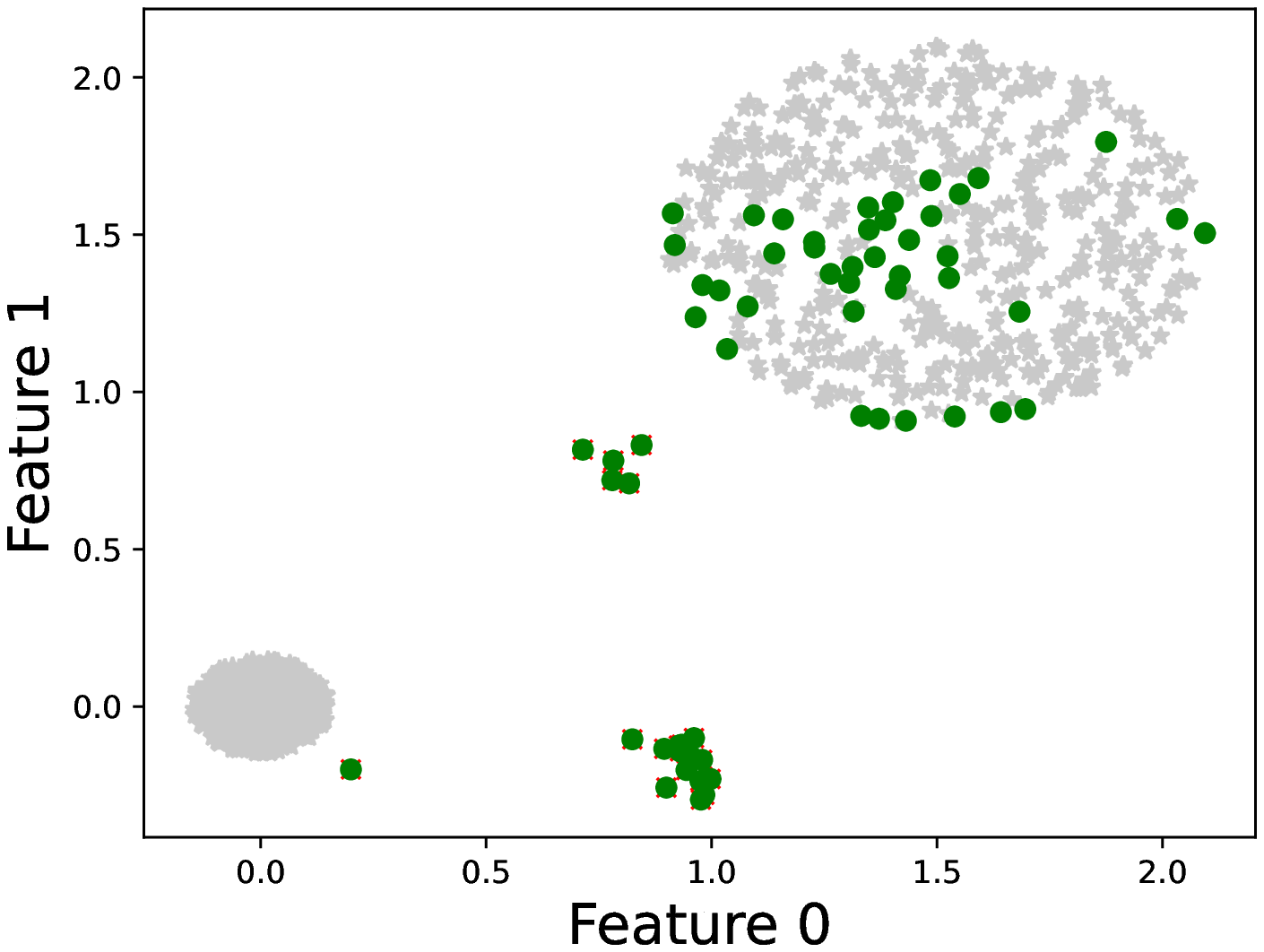}
    \caption{AIDA (V1).}
    \end{subfigure}
    \medskip
    \begin{subfigure}{0.4\textwidth}
    \centering
    \includegraphics[width=\textwidth]{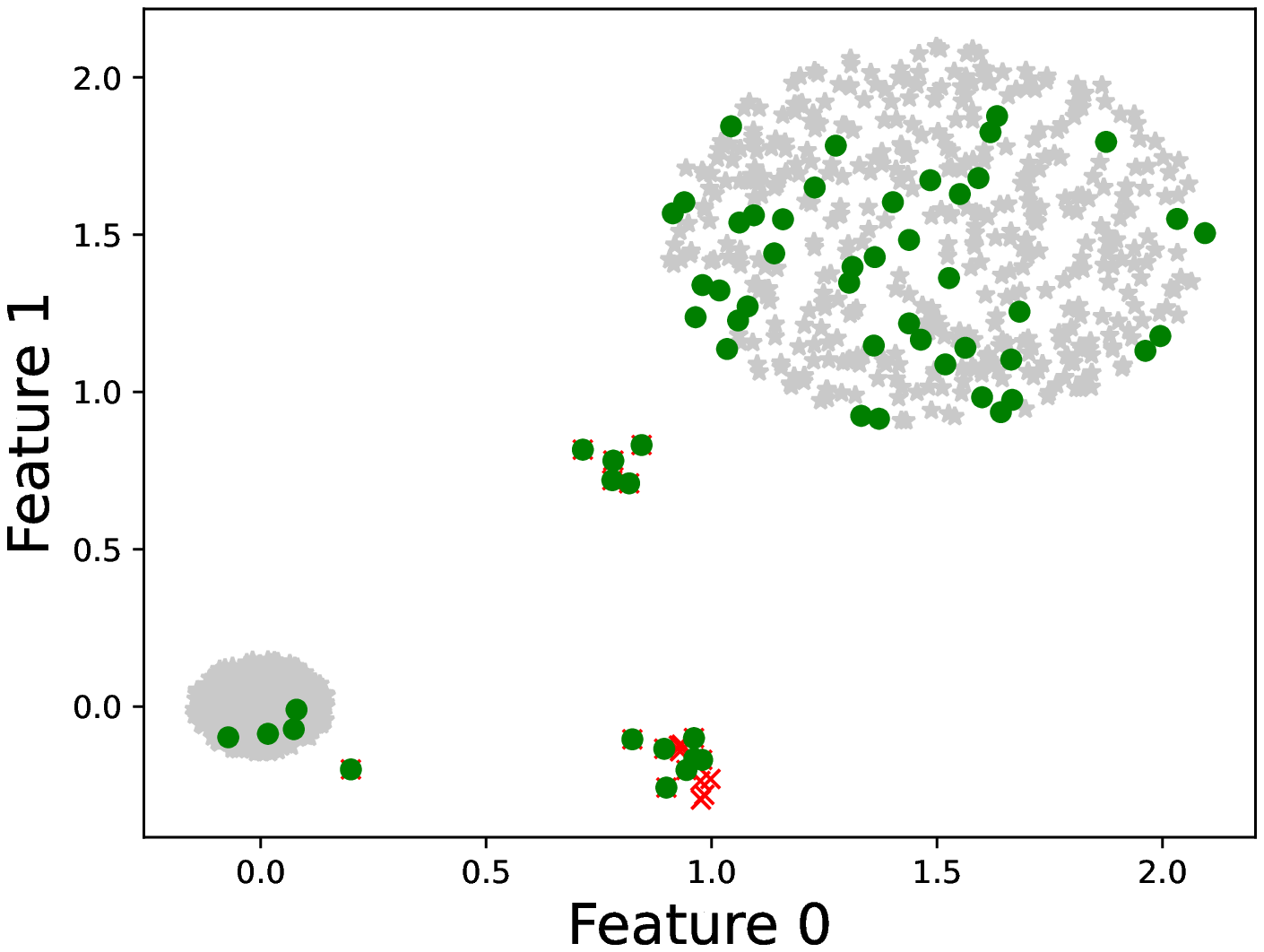}
    \caption{AIDA (E2).}
    \label{fig:example_2D_expectation2}
    \end{subfigure}
    \begin{subfigure}{0.4\textwidth}
    \centering
    \includegraphics[width=\textwidth]{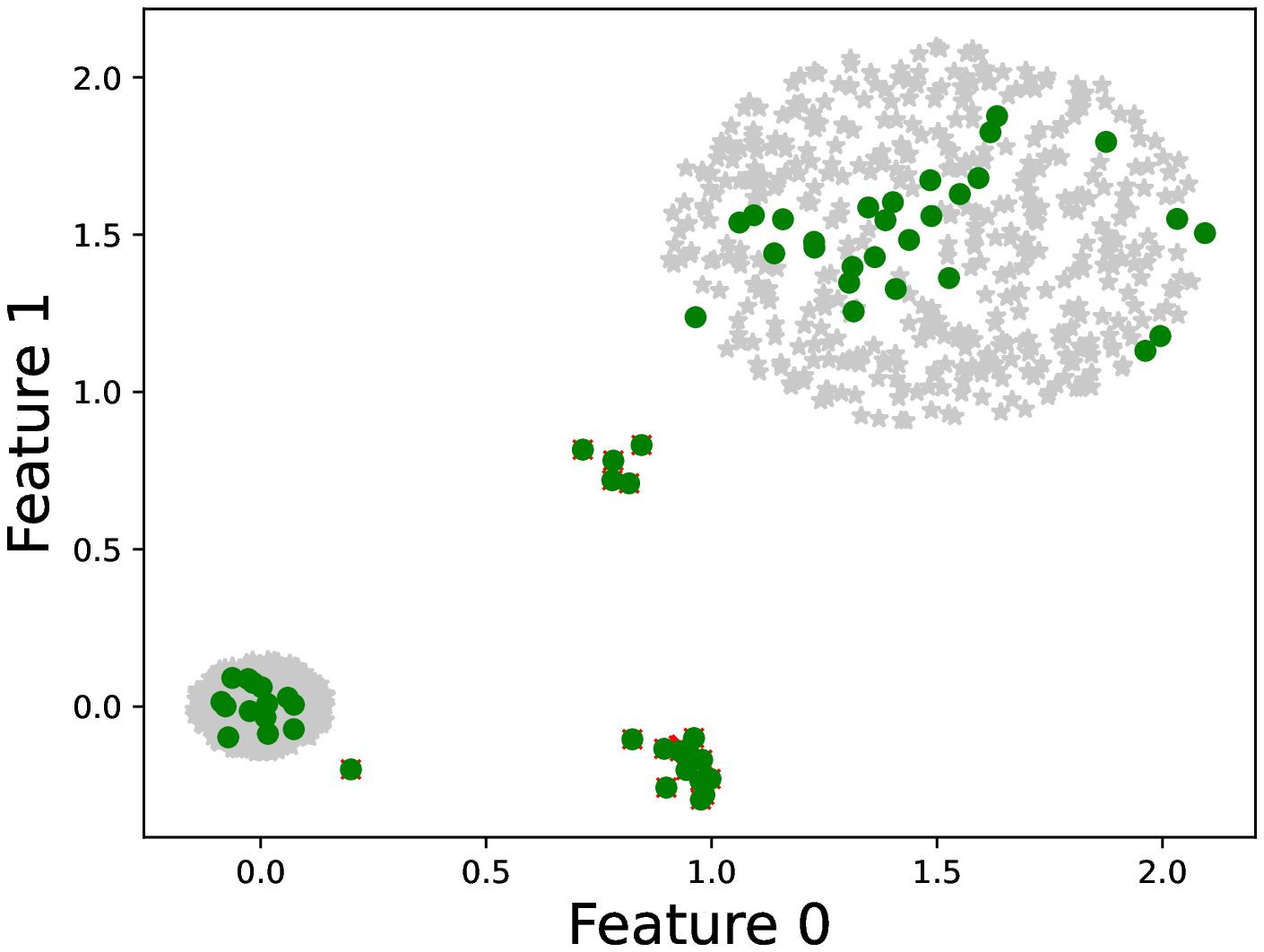}
    \caption{AIDA (V2).}
    \end{subfigure} 

    \caption{Comparison of the 60 most anomalous points detected by AIDA, iForest and LOF. For AIDA two different scores are used: expectation and variance. Inliers are marked with gray stars, detected outliers with green circles, and actual outliers with red crosses.}
    \label{fig:example_2D_outliers}
\end{figure}

We also present in \Cref{fig:example_2D_runtime} the computational time (in seconds) spent by AIDA on the test phase, described in \Cref{alg:aida_testing}, on several data sets of different dimensionality $d$ and number of observations $n$. In particular, we set $N = 100$, $\psi_{min} = 50$, $\psi_{max} = 512$ in \Cref{alg:aida_training} and use \Cref{eq:var_expo} as score function in \Cref{alg:aida_testing}. In \Cref{fig:example_2D_runtime_dim}, the number of observations was set to $n = 1000$, and in \Cref{fig:example_2D_runtime_n}, we fixed $d = 50$. From \Cref{fig:example_2D_runtime_dim,fig:example_2D_runtime_n} we observe that the computational time increases linearly both with the number of features and the number of observations, due to the use of subsamples.

\begin{figure}[ht]
    \centering
    \begin{subfigure}{0.4\textwidth}
    \centering
    \includegraphics[width=\textwidth]{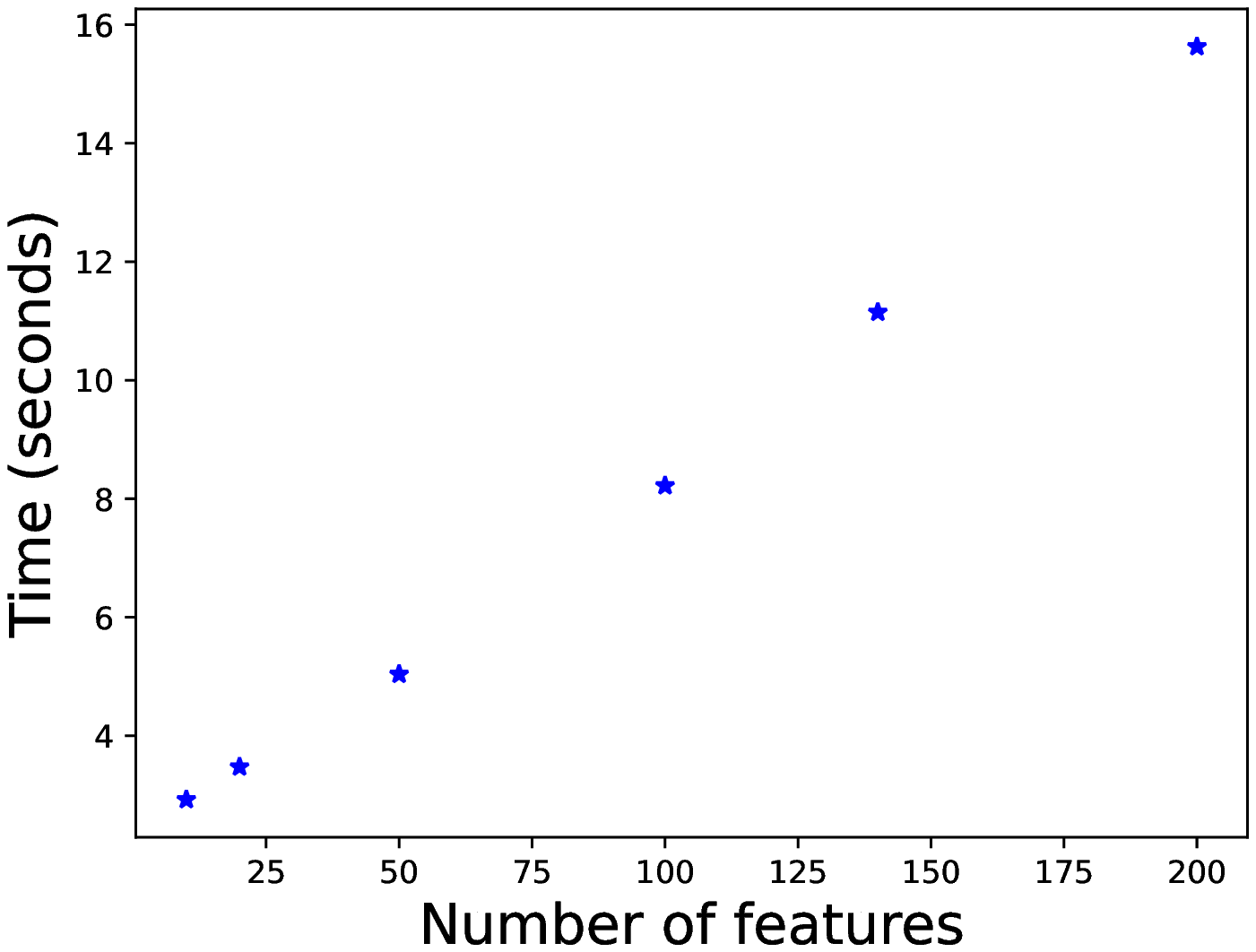}
    \caption{Runtime as a function of $d$.}
    \label{fig:example_2D_runtime_dim}
    \end{subfigure}
    \begin{subfigure}{0.4\textwidth}
    \centering
    \includegraphics[width=\textwidth]{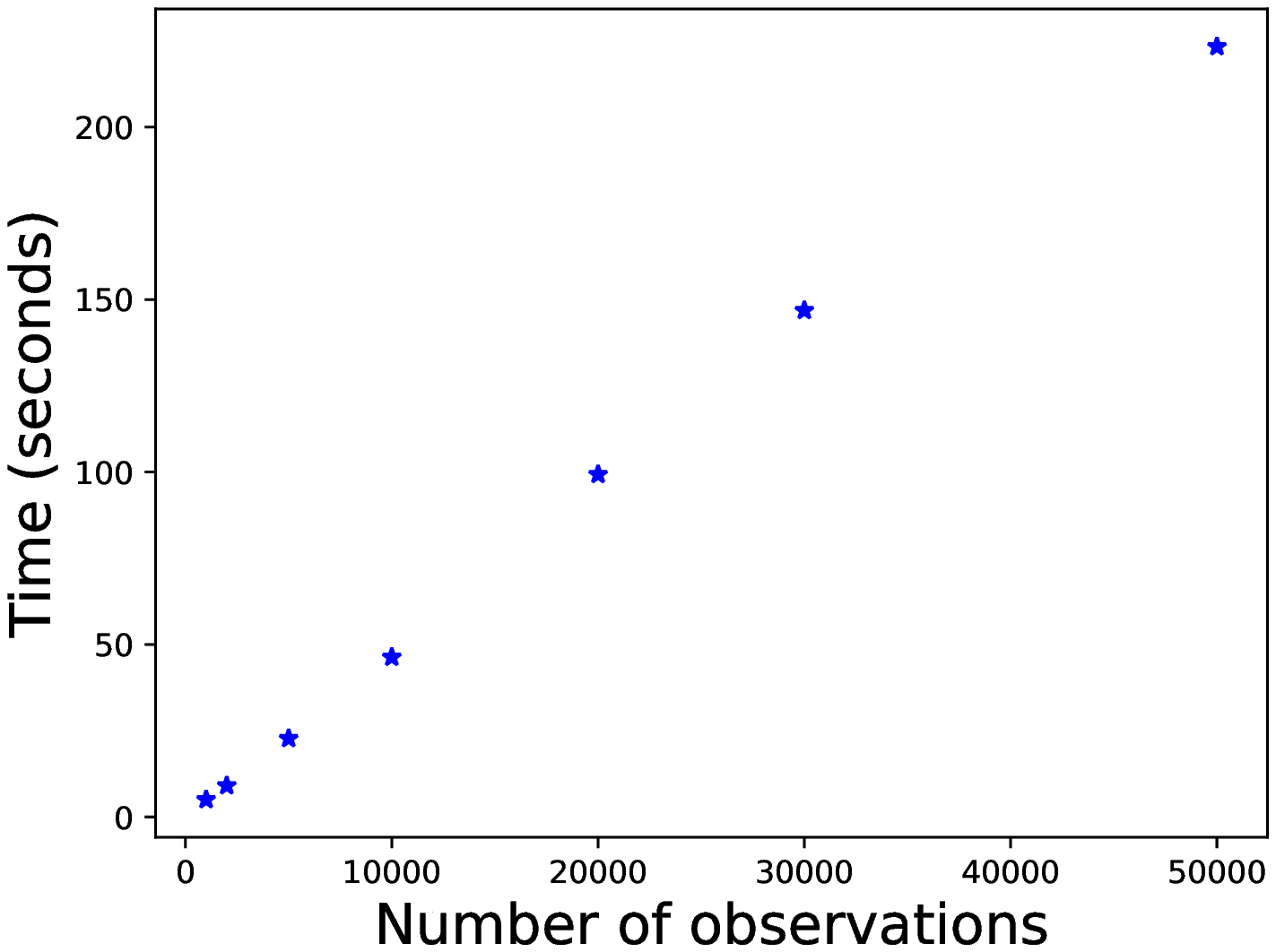}
    \caption{Runtime as a function of $n$.}
    \label{fig:example_2D_runtime_n}
    \end{subfigure}

    \caption{Computational times of the AIDA algorithm as a function of the dimensionality $d$ and the number of observations $n$. In the left plot (a) we fixed $n = 1000$, and, in the right plot (b), we set $d = 50$.}
    \label{fig:example_2D_runtime}
\end{figure}

\section{Explainability}\label{sec:explain}

In many practical contexts, the ability to explain why a certain observation is labelled an outlier is as important as the anomaly detection process itself. Especially in data sets with hundreds, or even thousands, of features, anomaly explanation can be a very complex and time-consuming task. Furthermore, in settings where only outliers generated by a specific mechanism are interesting (e.g. frauds or illegal transactions in financial data sets), having a preliminary understanding of which features characterize an outlier can serve as a filter to discard anomalies generated by other causes.

It is, however, difficult to extract explanations from distance-based methods, since they compress information from every feature into the distance metric \cite{panjei2022}. Exploring every possible subspace with the aim of finding subsets of features where the outliers are more remarkable is, of course, also not a viable option due to the curse of dimensionality.

In this article, we propose an explanation method for distance-based methods combining AIDA and the Simulated Annealing (SA) algorithm, that is commonly used in many other settings, such as global optimization and clustering (see \cite{kirkpatrick1983,ni2007,philipp2007}). Due to the partial inclusion of the annealing process in the explanation method, we call it the \textit{Tempered, Isolation-based eXplanation method} (TIX), described in \Cref{sec:tix}. We compare the inclusion of the SA acceptance criterion with the standard ``greedy'' approach \cite{mokoena2022} in \Cref{sec:sa_greedy}, and propose a possible refinement in \Cref{sec:refine}. In order to facilitate the interpretation of the results, we also propose the use of \textit{distance profile plots (DPP)}, which we define in \Cref{sec:dpp}.

\subsection{TIX algorithm}\label{sec:tix}

A good explanation method should be able to determine which features are most relevant to define outliers. In the context of the AIDA algorithm, this means finding the minimal feature subspace in which an outlier is easiest to isolate. However, due to the curse of dimensionality, it is computationally unfeasible to explore all the existing feature subspaces. One possibility to deal with this aspect is to use a so-called backward procedure in which, starting with the full feature space, we remove one feature at a time, and check whether the point of interest is easier to isolate in this reduced feature subspace. If that is the case, the chosen feature is deemed irrelevant and removed from the explanation process. Repeating this process until only the most relevant features are left is known as a ``greedy'' sequential search \cite{mokoena2022}.

Naturally, an accurate explanation method should aim at minimizing the number of important features, such that, if an outlier is equally easy to isolate in two different feature subspaces, the subspace with the least number of features should be preferred. For this reason, we propose a penalization mechanism which is based on the acceptance criterion of SA \cite{kirkpatrick1983}, so that explanations with only a few features receive a higher importance score than explanations with a larger number of features.

The procedure is as follows: given a potentially interesting outlier $X$ and a subsample $\bm{Y}_{\psi_i}$, for $i = 1,...,N$, we start by computing the score of $X$ with respect to $\bm{Y}_{\psi_i}$ using the full feature subspace $\mathcal{J} = \{1,...,d\}$, i.e. $f_{\mathcal{J}}(X) = score(X|\bm{Y}_{\psi_i},\mathcal{J})$. We randomly select an index $j$ from $\mathcal{J}$ and compute $f_{\mathcal{J}_{-j}}(X) =  score(X|\bm{Y}_{\psi_i},\mathcal{J}_{-j})$, where $\mathcal{J}_{-j}$ indicates that feature $j$ has been removed from $\mathcal{J}$. If $f_{\mathcal{J}_{-j}}(X) \geq f_{\mathcal{J}}(X)$, we set $\mathcal{J} = \mathcal{J}_{-j}$ and repeat the process.

On the other hand, if $f_{\mathcal{J}_{-j}}(X) < f_{\mathcal{J}}(X)$ we define the quantity
\begin{equation}\label{eq:accep_crite}
    p_j = \exp\left(\frac{f_{\mathcal{J}_{-j}}(X) - f_{\mathcal{J}}(X)}{f_{\mathcal{J}}(X) \cdot T}\right),
\end{equation}
where $T > 0$, and draw a uniform random variable $V \sim U(0,1)$. If $p_j \geq V$, we remove feature $j$ by setting $\mathcal{J} = \mathcal{J}_{-j}$. Otherwise, nothing changes. This process is repeated until a maximum number of iterations is reached, or until only one feature remains $(|\mathcal{J}| = 1)$. Each feature receives a score based on how many iterations of this process it has ``survived''. We refer to the number of iterations as the \textit{path length}. Relevant features should be more difficult to remove, and should therefore have a longer path length than irrelevant features. It is recommended to run this algorithm a fixed number of times $M$ in order to have a consistent estimate of the path length for each feature. \Cref{alg:tix} provides a pseudocode of the proposed TIX method.

\begin{algorithm} [!ht]
\caption{TIX: pseudocode.}\label{alg:tix}
\begin{algorithmic}
\State Load the potential outlier $X$.
\State Choose a score function from \Cref{eq:mean_expo} or \Cref{eq:var_expo}.
\State Set $\alpha$, $\bm{\omega}^{num}$ and $\bm{\omega}^{nom}$ equal to 1.
\For{$k = 1,...,M$}
    \For{$i = 1,...,N$}
        \State Set $\mathcal{J} = \{1,...,d\}$.
        \State Set $f_{\mathcal{J}}(X) = score(X|\bm{Y}_{\psi_i},\mathcal{J})$.
        \State Set $l = 0$.
        \State Set $T \sim U(T_{min},T_{max})$.
        \While{($l < L$) \texttt{or} ($|\mathcal{J}| > 1$)}
           \State Randomly select an index $j$ from $\mathcal{J}$.
           \State Set $f_{\mathcal{J}_{-j}}(X) = score(X|\bm{Y}_{\psi_i},\mathcal{J}_{-j})$.
           \If{$f_{\mathcal{J}_{-j}}(X) \geq f_{\mathcal{J}}(X)$}
             \State Set $\mathcal{J} = \mathcal{J}_{-j}$.
             \State \texttt{path\_length}$(j,i,k) = l$.
           \Else
             \State Compute $p_j$ using \Cref{eq:accep_crite}.
             \State Set $V \sim U(0,1)$.
             \If{$p_j > V$}
                \State Set $\mathcal{J} = \mathcal{J}_{-j}$.
                \State \texttt{path\_length}$(j,i,k) = l$.         
             \EndIf
           \EndIf
           \State Set $l = l +1$.
        \EndWhile \\
        \If{$|\mathcal{J}| > 1$}
            \For{each $j\in \mathcal{J}$}
                \State \texttt{path\_length}$(j,i,k) = l$ .
            \EndFor
        \EndIf    
    \EndFor
\EndFor\\

\For{$j = 1,...,d$}
    \State Aggregate the path lengths obtained over all subsamples and iterations.
\EndFor
\end{algorithmic}
\end{algorithm}

Notice that, since we are using the absolute score, it is possible that outlier scores are higher in high-dimensional settings, even if there exists a small subset of features where the sample could be easily isolated. This is due to the curse of dimensionality, by which the distance to the nearest and the furthest neighbours converges to the same value \cite{aggarwal2001}. This is shown in \Cref{fig:dpp_outlier_hics}, where we provide the DPs of an outlier in the HiCs data set 20.1 (for more details: \Cref{sec:hics}) using incremental feature spaces. Concretely, the top DP uses only the first feature, the second-top DP uses the first two features for the distance calculation, and so on, until the bottom DP, which uses the full feature space. From the results of \Cref{fig:dpp_outlier_hics}, it is clear that the outlier is easily isolated when we consider the first three features, and this is indeed how the outlier was generated \cite{keller2012}. 

Nonetheless, the same outlier is even easier to isolate when the full feature space is considered. Therefore, if we do not include the acceptance criterion of SA using \Cref{eq:accep_crite}, TIX would hardly remove any features, and the explanations would not be informative. This is a problem similar to that of model selection in regression models, where the goal is to find a parsimonious linear predictor \cite{Dunn2018}, i.e. the one with the desired explanatory power and the smallest number of features. In fact, adding more features in regression models reduces the in-sample estimation error, but it also generates overfitting, thus affecting generalization. Thus, in order to reduce the number of features, the complexity of the model must be penalized, for instance, by using criteria like the Akaike Information Criterion (AIC) \cite{kotzbreakthroughs}. One of the main advantages of AIC and similar metrics is that results do not need to be recalculated, making them computationally efficient. While the same concept of a fixed penalization could be applied to anomaly explanation methods, it is not clear how to define such penalization in practice. Hence, a random penalization based on the acceptance criterion of SA generates robust results in several diverse settings, avoiding the problem of defining a parameter whose values are not clear, and difficult to use in practice.

In any case, TIX contains a parameter $T$, the analogue of the temperature in the original SA algorithm, which affects the explanation results. From \Cref{eq:accep_crite}, it is clear that large values of $T$ increase the acceptance rate, and vice versa. Hence, it is desirable to find a value of $T$ that only maintains the most relevant features. For that purpose, we set $\Delta = (f_{\mathcal{J}_{-j}}(X) - f_{\mathcal{J}}(X))/f_{\mathcal{J}}(X)$ and redefine $T$ in terms of the relative score difference $\Delta$. In particular, given a specific value for $\Delta$, we look for the value of $T$ such that $e^{-\Delta/T} = 0.9$, which implies
\begin{equation}\label{eq:T_rule_thumb}
    T = \frac{\Delta}{\log(\frac{10}{9})}.
\end{equation}
\Cref{eq:T_rule_thumb} can be interpreted as the ``temperature'' such that the acceptance probability of a particular $\Delta$ is $0.9$. This effectively changes the problem from choosing $T$ into choosing $\Delta$, which we find easier to interpret. 

The probability threshold of $0.9$ is chosen to enhance interpretability. With such a high probability of acceptance, $\Delta$ should be given small values in order to maintain relevant features. In particular, we suggest setting $\Delta = 0.01$, so that a relative score difference of $1\%$ has a $90\%$ chance of being accepted. Another alternative, which alleviates the effects of a poor choice of $\Delta$, is to randomly select $\Delta$ in a given interval, as it is shown in \Cref{alg:tix} for $T$. We suggest $\Delta \sim U(\Delta_{min},\Delta_{max})$, with $\Delta_{min} = 0.01$ and $\Delta_{max} = 0.015$. We will use this particular setting for all the experiments considered in \Cref{sec:results}.

Finally, the best and worst-case time complexity statements of the TIX algorithm for a single observation are approximately $\mathcal{O}(M N (\log(\psi_m)+2) \psi_m d)$ and $\mathcal{O}(M N (\log(\psi_m)+2) \psi_m L)$, respectively, for $L > d$. If $T$ is too large, then features are always removed, and the condition $|\mathcal{J}| = 1$ is met in $d-1$ steps. In contrast, if $T$ is too small, it will be unlikely to remove any features and the algorithm will not stop until the maximum number of iterations $L$ is reached. Notice that it is not necessary to recompute all the distances at every iteration. Since only one feature is removed at a time, it is more efficient to compute the contribution of that feature to the distance metric, and remove it from the distances computed in the previous iteration, which has an average time complexity of $\mathcal{O}(\psi_m)$. Thus, the main bottleneck of the algorithm is sorting the distance values and computing the scores, which have time complexities of $\mathcal{O}(\log(\psi_m)\psi_m)$ and $\mathcal{O}(\psi_m)$, respectively.

\subsection{``Greedy'' approach vs. SA approach}\label{sec:sa_greedy}

We analyze the benefits of the SA acceptance criterion, as defined in \Cref{eq:accep_crite}, from a theoretical and practical perspective. We prove that, in the simple scenario of \Cref{eq:mean_uni}, the greedy approach fails to remove irrelevant features, even when their contribution to the outlier score is infinitesimally small. We also compare the performance of the TIX algorithm with and without the acceptance criterion with a simple synthetic example.

The setting is as follows: assume that we have the DP of a potential outlier $X^*$ with respect to a data set $\bm{X}_n$ of dimensionality $d$, i.e. DP$(X^*,\bm{X}_n)$. For simplicity, we further assume that the outlier score function is given by the opposite of \Cref{eq:mean_uni}---so that anomalous points have higher scores than inliers---that there are no nominal features and that $p = 1$ in \Cref{eq:distance}, i.e. we use the Manhattan distance.

Applying the methodology developed in \Cref{sec:gen_set}, the outlier score of $X^*$ with respect to the full feature space $\mathcal{J}$, denoted $f_{\mathcal{J}}(X^*)$, is calculated by using the sorted distances DP$(X^*,\bm{X}_n)$ as input in \Cref{eq:mean_uni}. Next, assume that the contribution of feature $j$ to the distance computation in \Cref{eq:distance} is such that $l_1(X^*,X_i|\mathcal{J}) = l_1(X^*,X_i|\mathcal{J}_{-j}) + \Delta x$, for $i=1,...,n$---where $\mathcal{J}_{-j}$ represents the feature subspace $\mathcal{J}$ without feature $j$, as in \Cref{sec:tix}---and $\Delta x > 0$. Since substraction of a constant  does not affect the ranks of the distances, we also have that DP$(X^*,\bm{X}_n|\mathcal{J}) = \text{DP}(X^*,\bm{X}_n|\mathcal{J}_{-j}) + \Delta x$. In this simple example, the difference between the outlier scores with and without feature $j$ is given by
\begin{equation}\label{eq:diff_score1}
    f_{\mathcal{J}_{-j}}(X^*) - f_{\mathcal{J}}(X^*) = \sum_{i=2}^{n-1} (Z_{i+1}-Z_i) \left(\frac{1}{Z_{i+1}-Z_1} - \frac{1}{Z_{i+1}-\Delta x-Z_1}\right),
\end{equation}
where $Z_i$ is the $i$-th sorted distance in DP$(X^*,\bm{X}_n|\mathcal{J})$, and $Z_1$ is always zero ($l_1(X^*,X^*) = 0$) in the context of the AIDA algorithm.

\Cref{eq:diff_score1} can be further simplified as 
\begin{equation}\label{eq:diff_score2}
    f_{\mathcal{J}_{-j}}(X^*) - f_{\mathcal{J}}(X^*) = -\Delta x \sum_{i=2}^{n-1} \frac{Z_{i+1}-Z_i}{(Z_{i+1}-Z_1)(Z_{i+1}-\Delta x-Z_1)}.
\end{equation}
Since $\Delta x >0$, it is clear that \Cref{eq:diff_score2} is always negative, regardless of $\Delta x$. Therefore, a greedy approach will never remove feature $j$, irrespective of how small $\Delta x$ is. On the other hand, the probability of removing the same feature with the SA approach converges to one as $\Delta x \rightarrow 0$, which can be easily verified by substituting \Cref{eq:diff_score2} into \Cref{eq:accep_crite}.

\begin{remark}
In data sets of very high dimensionality, the contribution of each individual feature to \Cref{eq:total_distance} will be small compared to the remaining $d-1$ features, which is another consequence of the curse of dimensionality. Hence, it is expected that a greedy approach will hardly remove any features in data sets where $d$ is large, which is precisely when an explanation method would be most important. In contrast, the SA approach employed by TIX is likely to remove any feature during the first iterations of \Cref{alg:tix} (a value of $\Delta x$ close to zero implies a probability of acceptance close to $1$ in \Cref{eq:accep_crite}), even those features that are actually relevant to the explanation process. Conversely, as the number of features decreases during the last stages of the TIX algorithm, it becomes more difficult to remove relevant features, while irrelevant features are still easy to discard. This is also the reason why it is recommended to run $\Cref{alg:tix}$ several times $(M>1)$. Otherwise, it is possible that the relevant features are removed first, resulting in inaccurate results.
\end{remark}

We illustrate these aspects with a synthetic example, which we label as the Cross data set, due to the shape of \Cref{fig:cross_example}. Concretely, we generate a data set of $n = 1000$ observations with different dimensionalities $d$, such that all the observations follow a uniform random distribution in the first $d-2$ features, and a single outlier is contained in the last two features (see \Cref{fig:cross_example}). Thus, from the point of view of the outlier, the number of irrelevant of features is $d-2$, and we expect an accurate explanation method to return the last two features as the most relevant ones. Notice that, given the shape of \Cref{fig:cross_example}, the outlier cannot be detected by looking at each of the last two features separately, hence the explanation results will not be accurate unless both features receive a high importance score.

\begin{figure}[ht]
    \centering
    \begin{subfigure}{0.4\textwidth}
    \centering
    \includegraphics[width=\textwidth]{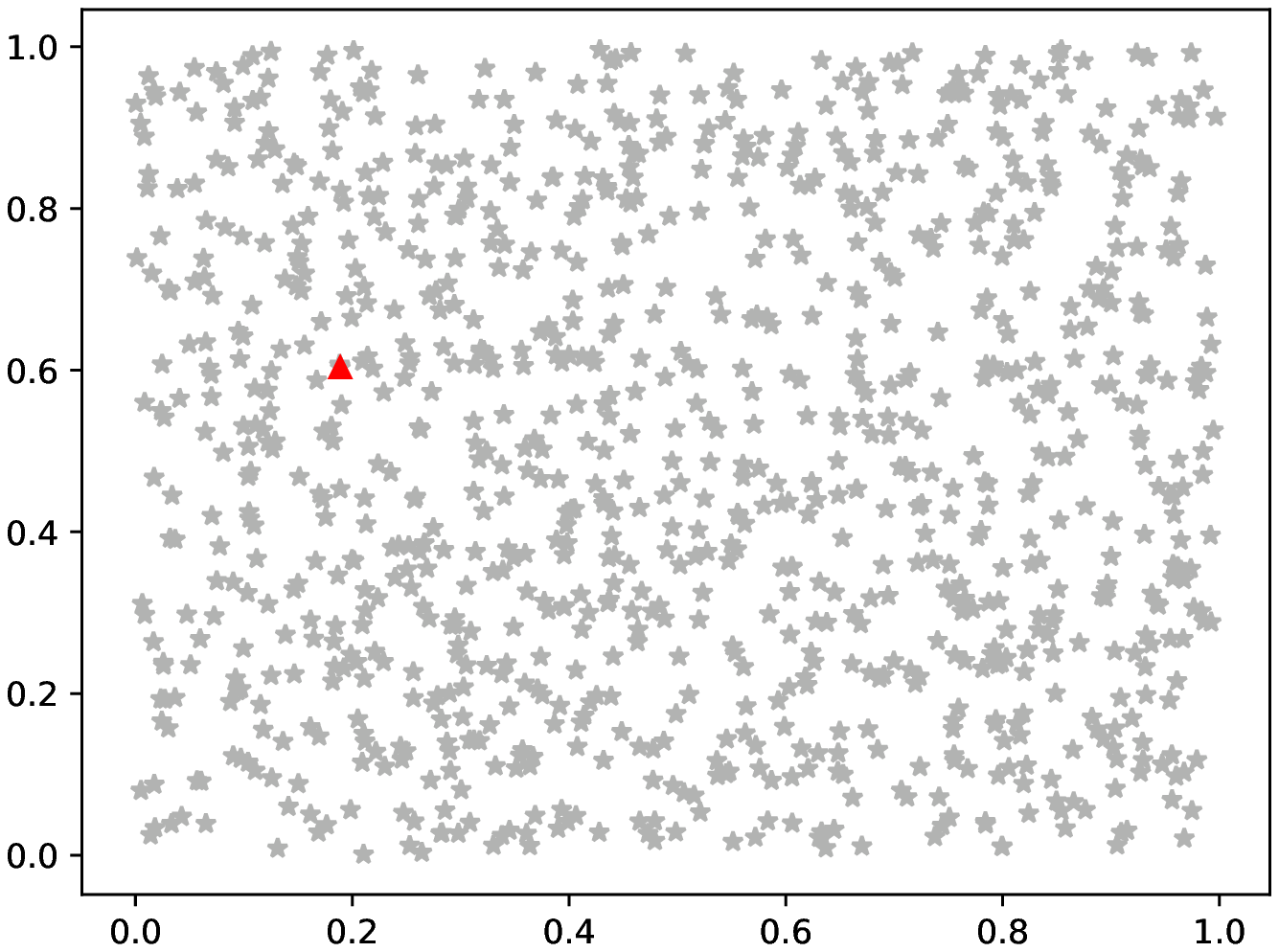}
    \caption{2D plot of two irrelevant features.}
    \end{subfigure}    
    \begin{subfigure}{0.4\textwidth}
    \centering
    \includegraphics[width=\textwidth]{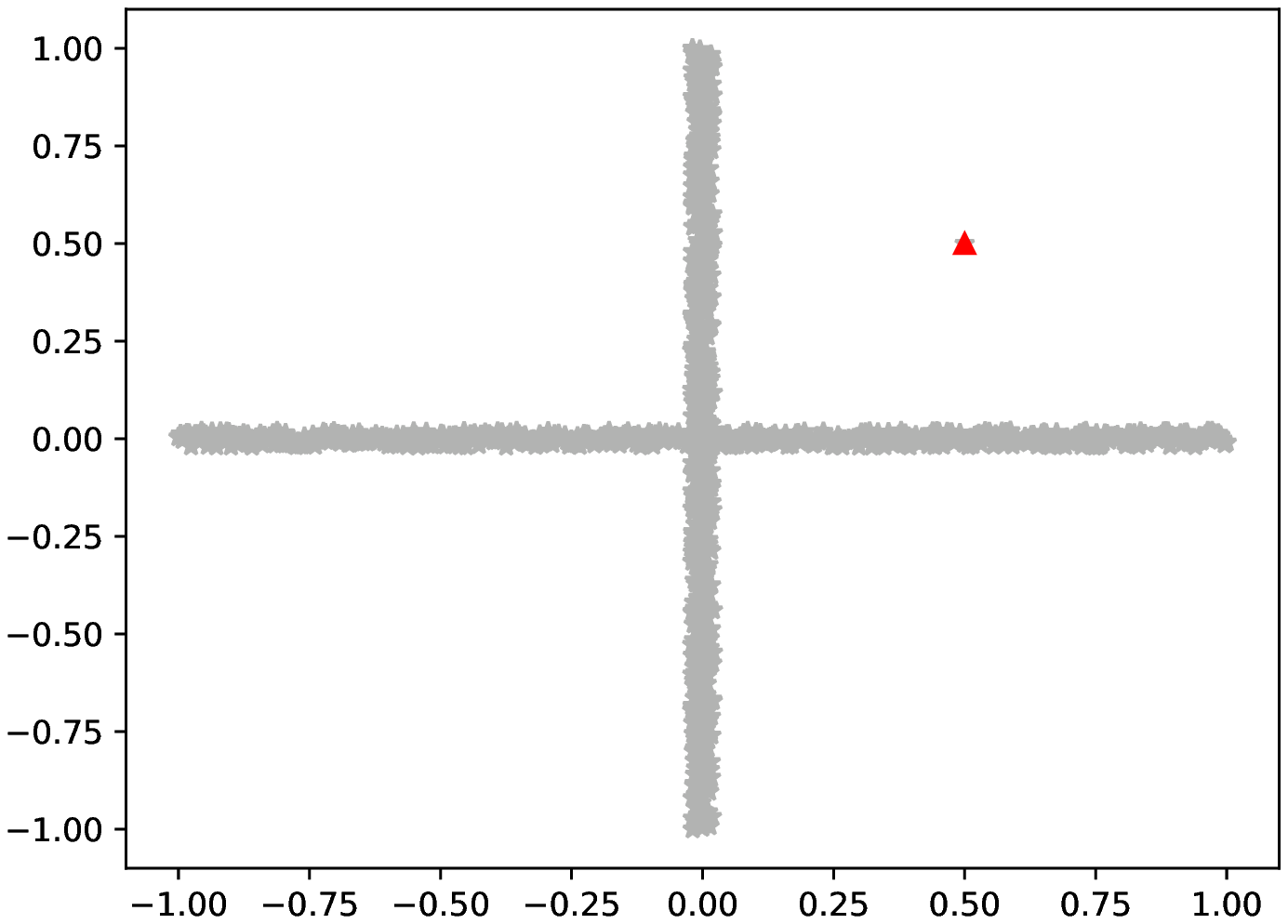}
    \caption{2D plot of the relevant features.}
    \end{subfigure}
    \caption{Plot of the Cross data set in two-dimensional projections of irrelevant features (left) and relevant features (right). Inliers are marked with gray stars, and the outlier with a red triangle.}
    \label{fig:cross_example}
\end{figure}

We test the TIX algorithm on this data set with and without the acceptance criterion (SA vs. greedy) for $d = 5$, $10$ , $20$, $30$, $40$, $50$ and $100$. In the SA approach, we set $T_{min} = 0.01$ and $T_{max} = 0.015$, as explained in \Cref{sec:tix}, with $M = 10$ in both approaches. We also set $N = 100$, $\psi_{min} = 50$ and $\psi_{max} = 512$ in \Cref{alg:aida_training}. The performance of the algorithms is measured in terms of the minimal feature subspace that contains the relevant features. That is, how many features need to be analyzed until the relevant features are found. For example, if the last two features receive the third and fifth highest scores, we need to analyze five features until we find the most relevant ones. The smaller the size of the minimal feature subspace, the more accurate are the explanation results. In this particular example, a minimal feature subspace of size two implies a perfect score for the explanation method.

The results are displayed in \Cref{tab:tix_comparison}, where we report the average minimal feature subspace, with its corresponding standard deviation, over 10 different executions\footnote{This is not the same as the number of iterations $M$.}. As expected from \Cref{eq:diff_score2}, the performance of the greedy approach quickly decays as the dimensionality increases. In contrast, the inclusion of the SA acceptance criterion yields perfect results for $d \leq 50$, since the relevant features were always found in every execution of the algorithm. Nonetheless, if we further increase the dimensionality, even the results obtained with the SA approach will start to deteriorate, as it is the case for $d = 100$.

In fact, the magnitude of the standard deviation in the SA approach clearly indicates that the results are not stable, and a higher $M$ is required. Increasing $M$ from $10$ to $100$ yields a minimal feature subspace of $2.4_{\pm 0.9}$, which is close to a perfect score. However, a large $M$ also makes the algorithm computationally expensive, thus in \Cref{sec:refine} we propose a refinement procedure which yields similar results at a reduced computational cost.

\begin{table}[ht]
  \ra{1.3}
  \caption{Size of the average minimal subspace returned by TIX with and without the acceptance criterion of \Cref{eq:accep_crite}.}
  \centering
  \begin{threeparttable}
    \begin{tabular}{lrr}\midrule\midrule 
       $d$ & \multicolumn{1}{c}{SA} & \multicolumn{1}{c}{Greedy} \\ \midrule
        $5$ & $2.0_{\pm 0.0}$ & $2.0_{\pm 0.0}$ \\ 
        $10$ & $2.0_{\pm 0.0}$ & $3.2_{\pm 0.4}$  \\ 
        $20$ & $2.0_{\pm 0.0}$ & $4.1_{\pm 1.1}$  \\
        $30$ & $2.0_{\pm 0.0}$ & $13.7_{\pm 2.4}$ \\ 
        $40$ & $2.0_{\pm 0.0}$ & $26.3_{\pm 1.9}$ \\ 
        $50$ & $2.0_{\pm 0.0}$ & $45.1_{\pm 2.3}$ \\ 
        $100$ & $15.5_{\pm 27.9}$ & $87.9_{\pm 3.1}$ \\ 
        \midrule\midrule
    \end{tabular}
  \end{threeparttable}
  \label{tab:tix_comparison}
\end{table}

\subsection{Refinement step}\label{sec:refine}

The TIX algorithm described in \Cref{sec:tix} can be embedded\footnote{In principle, this refinement step can be applied to any explanation method that returns numeric scores or ranks per feature.} in a recursive procedure to further improve the explanation results. Concretely, once the importance scores have been returned by TIX, instead of directly reporting these scores to the analyst, a further refinement can be done by first selecting the top $k$ features, and then reapplying the TIX algorithm using these relevant features only. This refinement step can be repeated for decreasing values of $k$ until a desired $k_{min}$ is reached.

\begin{algorithm}
\caption{Refinement step.}\label{alg:refine}
\begin{algorithmic}[1]
\State Load the potential outlier $X$.
\State Set $\mathcal{J}$ equal to the full feature space.
\State Set $k = d$.
\State Set $\beta>0$.
\While{$k \geq k_{min}$}
  \State Compute the importance scores of the features in $\mathcal{J}$ with the TIX algorithm (\Cref{alg:tix}). 
  \State Set $k = \max(\lfloor k/\beta \rfloor,k_{min})$.
  \State Set $\mathcal{J}$ equal to the $k$ most relevant features.
  \State Compute the final score of the removed features.
\EndWhile
\State Compute the final score of the remaining features.

\end{algorithmic}
\end{algorithm}

\Cref{alg:refine} shows the pseudocode of the refinement step. The crucial part is how to determine the importance score of the removed features, since it is not trivial how to aggregate scores from different iterations. One possibility is to use ranks as the final scores. In that case, the scores returned by the TIX algorithm can be used to determine the ranks of the removed features at each iteration. Another possibility is to modify the scores so that they are compatible between iterations. In particular, we suggest to add $d-k$ to the path lengths of each feature. The reasoning is the following: in order to go from $d$ to $k$ features in the TIX algorithm, $d-k$ is the minimum path length that must be covered. Furthermore, \Cref{alg:refine} reduces to \Cref{alg:tix} if $d/\beta < k_{min}$. Since $\beta$ controls the degree of the refinement process in \Cref{alg:refine}, we refer to it as the \textit{refinement rate}.

In \Cref{sec:hics}, we will test the performance of \Cref{alg:refine} for several values of $\beta$.

\subsection{Distance profile plot (DPP)}\label{sec:dpp}

Once the most important features have been returned by the explanation method, it is still the task of the analyst to determine how many features are actually relevant, or which combination of them best explains the outliers. For that purpose, visualization techniques such as 2D plots are especially popular due to their interpretability \cite{gupta2019}. Nonetheless, sometimes the interactions among features require us to consider more than two features at the same time. In that case, 2D plots are not able to capture the outlier behaviour, giving the incorrect impression that the plotted features are not relevant. 

To summarize the outlier information in subspaces using any number of features, we propose the Distance Profile Plot (DPP). In a DPP, several DPs are plotted together to find the most relevant outlier subspaces. Each DP corresponds to a specific number of features, that is to a specific subspace. For each subspace, the distances from the point of interest are represented using a boxplot, which allows for a quick grasp of their distribution.

Since the distance of a point from itself is 0, every point of interest will always be the first point on the left in the DP. The more such a point can be isolated from the others, the more the whiskers of the boxplot will tend to shrink away from it, while the interquartile range of the distances will tend to condense around the median distance. Conversely, a point that in a given subspace cannot be easily isolated will be touched by the whiskers, and the interquartile range will be larger. An example of DPP is presented in \Cref{fig:dpp_hics}, where we show the DPP of an outlier and of an inlier in the HiCs data set 20.1, which is described in \Cref{sec:hics}. In particular, the outlier is known to be anomalous only in the feature subspace containing the first three features \cite{keller2012}. 

The top DP in \Cref{fig:dpp_hics} corresponds to the distances computed using only one of these three features, the second-top DP uses two of these features, and so on. Thus, if we analyze the DPs in \Cref{fig:dpp_outlier_hics} from top to bottom, we observe that the first and second features alone are not relevant to explain the outlier behaviour, since the first point on the left is touched by the whiskers of the boxplots. It is only when we reach the third DP in \Cref{fig:dpp_outlier_hics} that the point gets isolated, as the left whisker moves away, indicating that the feature subspace composed of the first three features could be relevant to explain the anomalous observation (and that is indeed how the outlier was generated). Conversely, the DPP plot of the inlier displayed in \Cref{fig:dpp_inlier_hics} shows that this point is not easy to isolate in the same feature subspace.

A consequence of the curse of dimensionality is that the DPs of the outlier and the inlier are very similar when the number of features becomes large. In fact, the first point on the left in the DPP of \Cref{fig:dpp_inlier_hics} gradually becomes easier to isolate as the number of features $d$ increases, in line with the fact that the distance to the nearest and the furthest neighbours converges to the same value for large $d$ \cite{aggarwal2001}. This is connected to the deterioration of the greedy approach discussed in \Cref{sec:sa_greedy}.

In contrast, we observe a sharp change in the DPP of \Cref{fig:dpp_outlier_hics} when all the relevant features are included, instead of a gradual increment of the distance between the left-fringe and its nearest neighbours. Thus, sharp changes in the DPPs are associated to relevant features, while gradual distance increments are due to the curse of dimensionality and indicate irrelevant features.

\begin{figure}[ht]
    \centering
    \begin{subfigure}{0.45\textwidth}
    \centering
    \includegraphics[width=\textwidth]{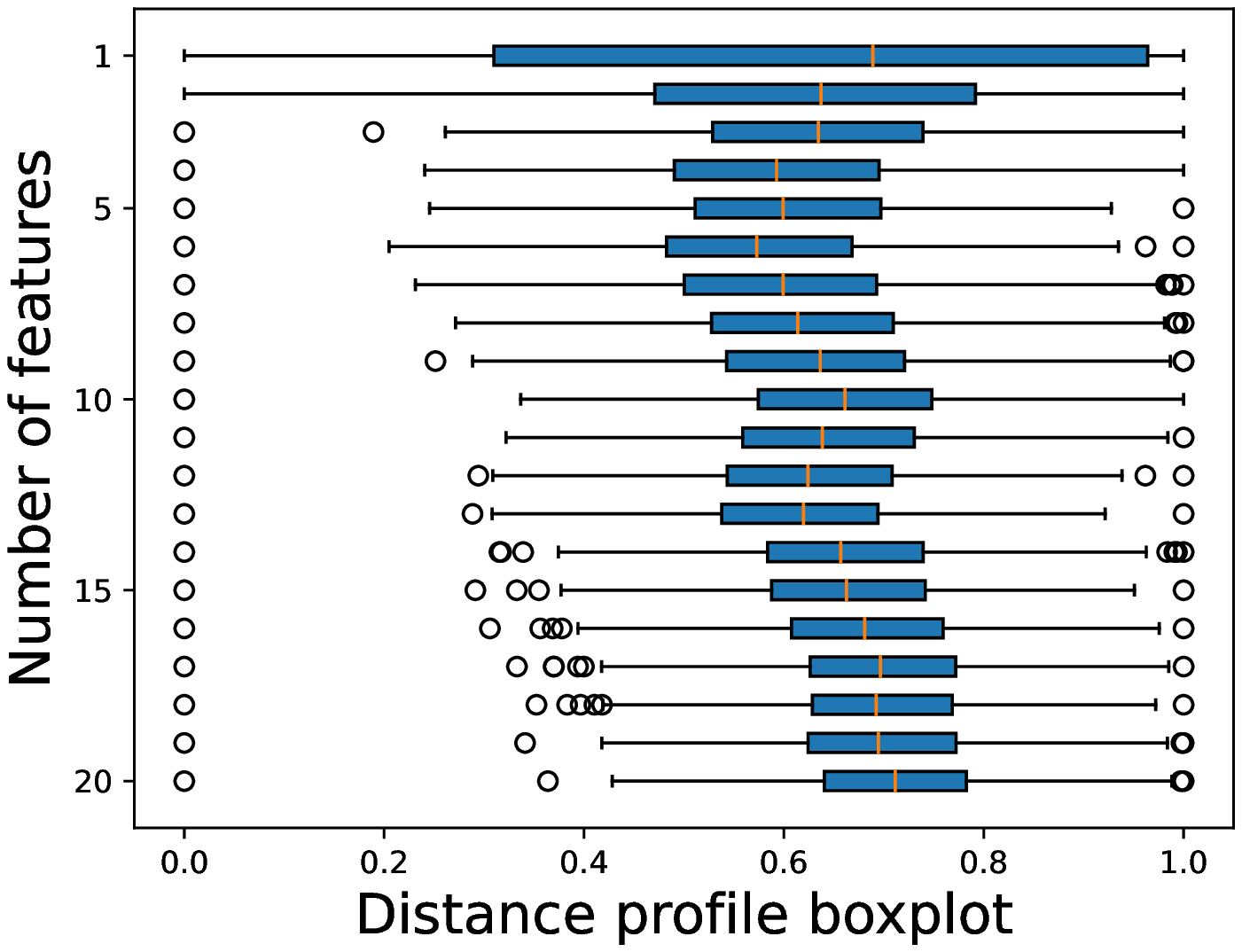}
    \caption{DPP of an outlier.}
    \label{fig:dpp_outlier_hics}
    \end{subfigure}    
    \begin{subfigure}{0.45\textwidth}
    \centering
    \includegraphics[width=\textwidth]{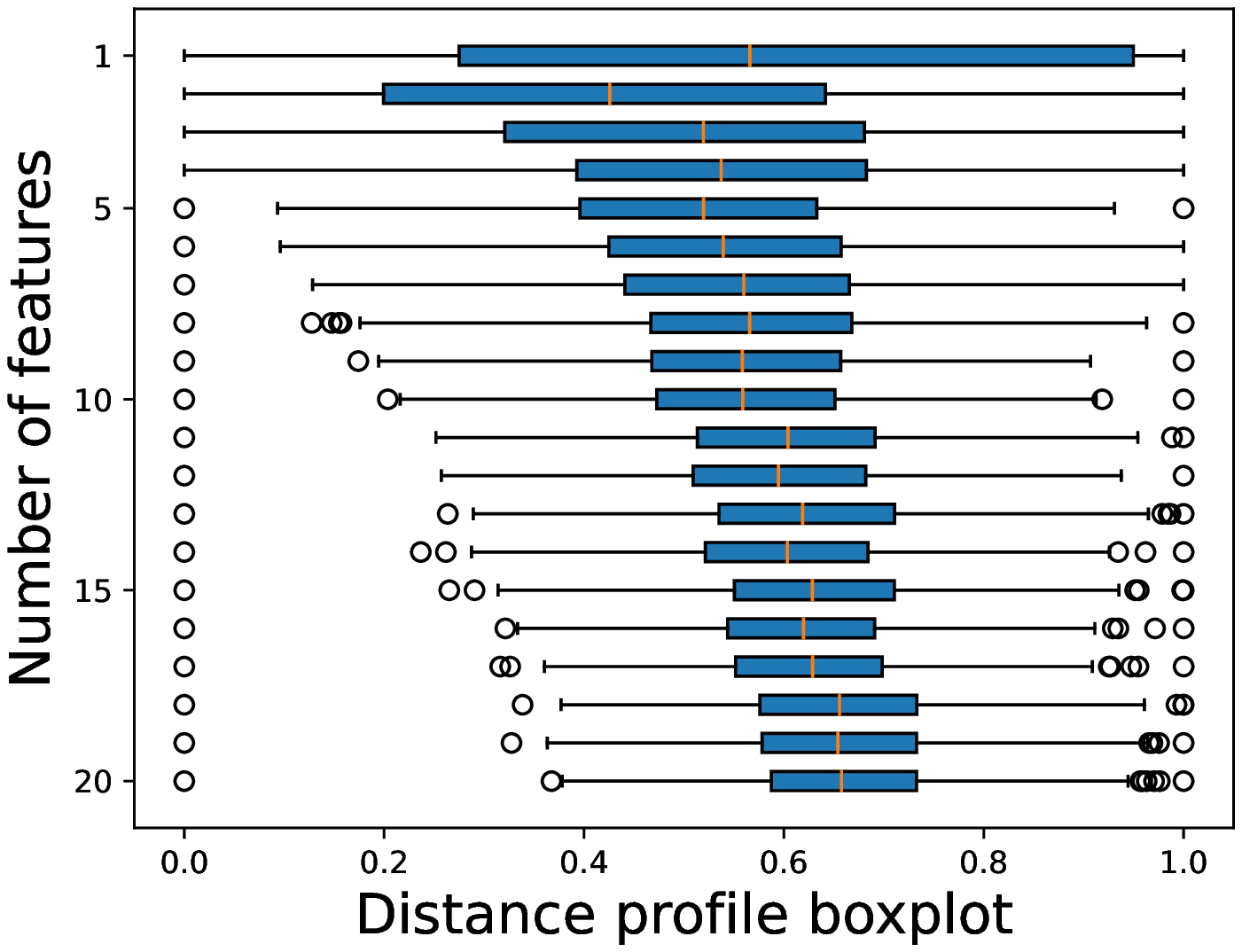}
    \caption{DPP of an inlier.}
    \label{fig:dpp_inlier_hics}
    \end{subfigure}
    \caption{Distance profile plot of an outlier (left) and an inlier (right) in the HiCs data set 20.1.}
    \label{fig:dpp_hics}
\end{figure}

In addition, while in \Cref{fig:dpp_hics} we have added the features in the order they appear in the data set, in the context of explanation algorithms, only the most relevant features should be used, so that the top DP corresponds to the most relevant feature, the second-top DP to the two most relevant features, and so on.
 
\section{Numerical results}\label{sec:results}

We test the performance of the proposed AIDA algorithm using artificial and empirical data sets, and compare it with several state-of-the-art anomaly detection methods: iForest, isolation using Nearest Neighbour Ensemble (iNNE) \cite{bandaragoda2018}, LOF and average kNN (AvgKNN) \cite{angiulli2002}. iForest is an isolation-based method, while LOF and AvgKNN are distance-based methods. iNNE, like AIDA, is a combination of both these concepts. Hence, we consider these models as good benchmarks to test AIDA.

In all experiments, we consider two settings for $\alpha$ in \Cref{eq:mean_expo,eq:var_expo}: $\alpha = 1$ and $\alpha \sim U(0.5,1.5)$. In the latter case, each subsample $\bm{Y}_{\psi_j}$, for $j=1,...,N$, has associated a value of $\alpha$ in the given interval. The reasoning is the same as the one given for $\Delta$ in \Cref{sec:tix}: randomizing $\alpha$ within a reasonable interval diminishes the risk of making a poor choice. We test both settings with the proposed outlier scores of \Cref{eq:mean_expo,eq:var_expo}, for a total of 4 different AIDA configurations. We use the letters E and V to indicate the score function, and the indicators 1 and R to indicate whether we use $\alpha = 1$ or a randomized alpha. For example, AIDA (VR) in \Cref{tab:auc_hics} refers to the AIDA algorithm using the variance as the score function with a randomized choice for $\alpha$.

Additionally, we set $N = 100$, $\psi_{min} = 50$ and $\psi_{max} = 512$. If the data set has dimensionality $d > 5$, we use feature bagging as described in \cite{lazaveric2005}. Otherwise, we use the full feature space. The aggregation of the scores over different subsamples is done using the \textit{Average of Maximum (AOM)} function, with the number of subsamples per bucket equal $q = 5$, as suggested in Section 4.3 of \cite{aggarwal2015}, for a total of 20 buckets. Regarding the distance metric, we use the Manhattan distance with all weights $\omega$ equal to one for AIDA, LOF and AvgKNN. Moreover, we set the number of neighbours to $k = \min(20,0.05\cdot n)$ in LOF and AvgKNN. With respect to iForest, we choose the number of trees equal to $100$ and the subsampling size to $256$. For iNNE, we set the number of trees to $100$ and the number of samples to $8$ \cite{bandaragoda2018}. Furthermore, since AIDA, iForest and iNNE are random algorithms, we report the average AUC over $10$ different runs, with their respective standard deviations\footnote{In some articles, it is common to test each algorithm for several parameter configurations and report the best performance (e.g., \cite{bandaragoda2018,pang2015}). However, as noted in \cite{aggarwal2015}, in practice it is not possible to know in advance whether a specific choice will yield good results. Thus, we have chosen typical parameter choices for each model and fixed them for all tests.}.

For the TIX method, we always consider \Cref{eq:var_expo} with $\alpha = 1$ as score function, and use the full feature space, so that the aggregation of the path lengths over different subsamples is consistent. Furthermore, the explanation method is also executed $10$ times to have a robust estimator of the expected path lengths, with $L = 50 \cdot d$ in \Cref{alg:tix}.

Finally, distance-based methods are sensitive to the scale of the numerical values, and this can introduce a serious bias in the results,  towards features with the largest magnitude \cite{james2021}. Hence, we normalize the empirical data sets (the artificial data sets are already normalized) using Z-scores, so that each feature contributes equally to the distance metric.

\begin{remark}
The algorithms were implemented in C++ using the g++ compiler (version 9.4.0) and they are available in the GitHub repository: \url{https://github.com/LuisSouto/AIDA}. Experiments were run using an Intel(R) Core(TM) i7-7700HQ CPU @ 2.80GHz processor. 
\end{remark}

\subsection{HiCs data sets}\label{sec:hics}

For the artificial data, we consider the data sets from \cite{keller2012}, to which we refer for a detailed description. We label these data sets as the \textit{HiCs data sets}, since these examples were constructed to illustrate the performance of the HiCs algorithm. What makes the HiCs data sets challenging is that the outliers are hidden in multidimensional subspaces of dimension at least two, and up to five. Each outlier looks like an inlier in any other subspace, therefore the level of irrelevant features for a particular anomaly is very high. We refer to the number of features that characterizes an outlier as $r$, so that, if an outlier is defined by a feature subspace consisting of three features, then $r=3$. In the HiCs data sets, $r$ can take values from $2$ to $5$.

Another advantage of using these data sets is that we also know which features are relevant for each outlier, providing a useful benchmark for the TIX algorithm. There are a total of $21$ data sets, consisting of $3$ data sets of dimensionality $10$, $20$, $30$, $40$, $50$, $75$ and $100$, respectively, with a constant sample size of $n = 1000$. We give to each data set a label consisting of its dimensionality and its version number. For example, the second data set with $d = 30$ is labelled as \textit{HiCs 30.2}. 

\subsubsection{Anomaly detection in the HiCs data sets}

The comparison between the different models is presented in \Cref{tab:auc_hics}, which clearly indicates the suitability of the AIDA algorithm in detecting multidimensional outlier subspaces. In particular, the best performances---marked in bold numbers---are always obtained using the variance score, with the best model using a randomized $\alpha$. Moreover, iForest and iNNE systematically return the lowest AUC (Area Under the Curve, see, e.g., \cite{james2021}), showing that they are not suitable for detecting outliers in multidimensional subspaces. As far as the iForest algorithm is concerned, this was expected due to \Cref{prop:hidden_prob} in \Cref{sec:app_proof}.

\begin{table}[ht]
  \ra{1.3}
  \caption{AUC obtained in the HiCs data sets with the different anomaly detection models. The variants of the AIDA algorithm are labelled according to the score function used---Expectation (E) or Variance (V)---and the choice of $\alpha$ in \Cref{eq:mean_expo,eq:var_expo} ($\alpha = 1$ (1) or random choice (R)).}
  \centering
  \resizebox{\columnwidth}{!}{
  \begin{threeparttable}
    \begin{tabular}{lcccccccc}\midrule\midrule 
        & AIDA (E1) &  AIDA (ER) & AIDA (V1) &  AIDA (VR) & iForest & iNNE & LOF  & AvgKNN\\ \midrule
        Hics 10.1 & \bm{$1.000_{\pm 0.001}$} & \bm{$1.000_{\pm 0.000}$} & \bm{$1.000_{\pm 0.000}$} & \bm{$1.000_{\pm 0.000}$} & $0.951_{\pm 0.007}$ & $0.901_{\pm 0.014}$ &  $0.993$  & $0.998$\\ 
        Hics 10.2 & $0.999_{\pm 0.001}$ & \bm{$1.000_{\pm 0.000}$} & \bm{$1.000_{\pm 0.000}$} & \bm{$1.000_{\pm 0.000}$} & $0.945_{\pm 0.010}$ & $0.889_{\pm 0.009}$ & $0.991$  & $0.995$ \\ 
        Hics 10.3 & $0.995_{\pm 0.002}$ & $\bm{0.996_{\pm 0.003}}$ & \bm{$0.998_{\pm 0.001}$} & \bm{$0.998_{\pm 0.001}$} & $0.859_{\pm 0.010}$ & $0.820_{\pm 0.017}$ & $0.975$  & $0.975$ \\ 
        Hics 20.1 & $0.874_{\pm 0.015}$ & $0.868_{\pm 0.026}$ & $\bm{0.910_{\pm 0.013}}$ & \bm{$0.920_{\pm 0.016}$} & $0.741_{\pm 0.018}$ & $0.745_{\pm 0.005}$ & $0.817$  & $0.836$  \\ 
        Hics 20.2 & $0.929_{\pm 0.011}$ & $0.928_{\pm 0.023}$ & $\bm{0.953_{\pm 0.008}}$ & \bm{$0.959_{\pm 0.006}$} & $0.777_{\pm 0.019}$ & $0.736_{\pm 0.011}$ & $0.863$  & $0.843$  \\
        Hics 20.3 & $0.947_{\pm 0.012}$ & $0.949_{\pm 0.015}$ & $\bm{0.970_{\pm 0.007}}$ & \bm{$0.974_{\pm 0.007}$} & $0.814_{\pm 0.017}$ & $0.762_{\pm 0.013}$ & $0.881$  & $0.869$  \\ 
        Hics 30.1 & $0.828_{\pm 0.029}$ & $0.825_{\pm 0.024}$ & $\bm{0.879_{\pm 0.016}}$ & \bm{$0.891_{\pm 0.015}$} & $0.722_{\pm 0.015}$ & $0.693_{\pm 0.008}$ & $0.731$  & $0.739$  \\ 
        Hics 30.2 & $0.852_{\pm 0.021}$ & $0.828_{\pm 0.032}$ & \bm{$0.893_{\pm 0.009}$} & \bm{$0.893_{\pm 0.011}$} & $0.678_{\pm 0.021}$ & $0.665_{\pm 0.010}$ & $0.733$  & $0.748$  \\
        Hics 30.3 & $0.860_{\pm 0.016}$ & $0.871_{\pm 0.011}$ & $\bm{0.911_{\pm 0.015}}$ & \bm{$0.923_{\pm 0.014}$} & $0.709_{\pm 0.013}$ & $0.685_{\pm 0.009}$ & $0.769$  & $0.763$ \\ 
        Hics 40.1 & $0.742_{\pm 0.022}$ & $0.750_{\pm 0.028}$ & $\bm{0.829_{\pm 0.015}}$ & \bm{$0.840_{\pm 0.016}$} & $0.645_{\pm 0.016}$ & $0.641_{\pm 0.010}$ & $0.726$  & $0.696$  \\ 
        Hics 40.2 & $0.768_{\pm 0.023}$ & $0.750_{\pm 0.026}$ &  $\bm{0.837_{\pm 0.014}}$ & \bm{$0.846_{\pm 0.009}$} & $0.608_{\pm 0.011}$ & $0.587_{\pm 0.013}$ & $0.685$  & $0.650$  \\ 
        Hics 40.3 & $0.757_{\pm 0.027}$ & $0.701_{\pm 0.022}$ & \bm{$0.824_{\pm 0.012}$} & $\bm{0.793_{\pm 0.015}}$ & $0.695_{\pm 0.016}$ & $0.680_{\pm 0.006}$ & $0.732$  & $0.718$  \\
        Hics 50.1 & $0.724_{\pm 0.025}$ & $0.723_{\pm 0.026}$ & $\bm{0.802_{\pm 0.009}}$ & \bm{$0.810_{\pm 0.021}$} & $0.611_{\pm 0.021}$ & $0.601_{\pm 0.007}$ & $0.679$  & $0.649$  \\
        Hics 50.2 & $0.716_{\pm 0.026}$ & $0.725_{\pm 0.021}$ & $\bm{0.802_{\pm 0.015}}$ & \bm{$0.815_{\pm 0.011}$} & $0.662_{\pm 0.011}$ & $0.651_{\pm 0.005}$ & $0.737$  & $0.708$  \\
        Hics 50.3 & $0.718_{\pm 0.020}$ & $0.716_{\pm 0.021}$ & $\bm{0.778_{\pm 0.011}}$ & \bm{$0.794_{\pm 0.017}$} & $0.630_{\pm 0.016}$ & $0.626_{\pm 0.008}$ & $0.670$  & $0.664$  \\
        Hics 75.1 & $0.616_{\pm 0.024}$ & $0.600_{\pm 0.016}$ & $\bm{0.672_{\pm 0.014}}$ & \bm{$0.675_{\pm 0.011}$} & $0.582_{\pm 0.011}$ & $0.582_{\pm 0.007}$ & $0.620$  & $0.604$  \\
        Hics 75.2 & $0.633_{\pm 0.017}$ & $0.634_{\pm 0.022}$ & $\bm{0.694_{\pm 0.015}}$ & \bm{$0.705_{\pm 0.015}$} & $0.586_{\pm 0.008}$ & $0.578_{\pm 0.005}$ & $0.631$  & $0.600$  \\
        Hics 75.3 & $0.608_{\pm 0.026}$ & $0.601_{\pm 0.029}$ & $\bm{0.673_{\pm 0.014}}$ & \bm{$0.685_{\pm 0.022}$} & $0.595_{\pm 0.016}$ & $0.586_{\pm 0.005}$ & $0.641$  & $0.615$  \\
        Hics 100.1 & $0.604_{\pm 0.022}$ & $0.603_{\pm 0.019}$ & $\bm{0.649_{\pm 0.015}}$ & \bm{$0.663_{\pm 0.015}$} & $0.578_{\pm 0.016}$ & $0.578_{\pm 0.004}$ & $0.622$  & $0.599$  \\
        Hics 100.2 & $0.575_{\pm 0.021}$ & $0.573_{\pm 0.029}$ & $\bm{0.615_{\pm 0.008}}$ & \bm{$0.632_{\pm 0.014}$} & $0.558_{\pm 0.014}$ & $0.574_{\pm 0.006}$ & $0.587$  & $0.590$  \\
        Hics 100.3 & $0.612_{\pm 0.015}$ & $0.606_{\pm 0.018}$ & $\bm{0.663_{\pm 0.014}}$ & \bm{$0.676_{\pm 0.008}$} & $0.573_{\pm 0.017}$ & $0.574_{\pm 0.005}$ & $0.613$  & $0.598$  \\
        \midrule\midrule
    \end{tabular}
  \end{threeparttable}
  }
  \label{tab:auc_hics}
\end{table}

On the other hand, no algorithm is able to provide highly satisfactory results for the most difficult cases, mainly those with $75$ and $100$ features. We proved in \Cref{prop:hidden_prob} that iForest has a very low probability of finding the relevant feature subspaces when the number of irrelevant features is large, thus this was an expected result. For the distance-based methods (including AIDA), the curse of dimensionality ``dilutes'' the contribution of each feature to the distance metric, resulting in a loss of discrimination between outliers and inliers \cite{aggarwal2001}. In those challenging cases, coupling the anomaly detection algorithm with an efficient subspace search method seems to be a viable choice for achieving accurate results \cite{keller2012}. Another alternative is to explore the effect of different distance metrics \cite{aggarwal2001}, since some of them have been shown to produce very diverse results \cite{zimek2014}.

\subsubsection{Anomaly explanation in the HiCs data sets}

We now present the explanation results of the TIX algorithm. Since we know beforehand which features define the outliers, we can measure the performance of TIX by the size of the minimal feature subspace that contains all the relevant features. For example, if the outlier is characterized by a combination of three features, a minimal feature subspace of size three means that the algorithm has successfully found the important features without adding any noise. If, on the other hand, the minimal subspace contains five features, two irrelevant features need to be checked before finding the relevant subspace.

In order to measure the impact of the refinement step, we test several values of the refinement rate $\beta$ in \Cref{alg:refine} under similar computational constraints. That is, since a smaller $\beta$ leads to more iterations in \Cref{alg:refine}, we modify the number of iterations in \Cref{alg:tix} so that each version takes approximately the same amount of time. Otherwise, it could be argued that \Cref{alg:refine} leads to better results due to the extra computations. Specifically, we set $k_{min} = 10$ and test $\beta \in \{1.5,2,10\}$, with $M = 20$ for the case $\beta = 10$, and adapting $M$ to the other values of $\beta$ with a grid search until the computational times are similar.

The results can be seen in \Cref{tab:tix_hics}, where $r$ denotes the size of the outlier feature subspace, and entries with ``--'' indicate that the data set does not contain outliers in feature subspaces of that dimensionality. Each entry contains the average minimal subspace over all the outliers characterized by a particular value of $r$, so that an entry value equal to $r$ indicates a perfect score.

\begin{table}[ht]
  \ra{1.3}
  \caption{Size of the average minimal subspace returned by TIX on several HiCs data sets and different outlier subspaces.}
  \centering
  \resizebox{\columnwidth}{!}{
  \begin{threeparttable}
    \begin{tabular}{lrrrrrrrrrrrr}\midrule\midrule 
       & \multicolumn{4}{c}{$\beta = 1.5$} & \multicolumn{4}{c}{$\beta = 2$} & \multicolumn{4}{c}{$\beta = 10$} \\ \cmidrule(lr){2-5}\cmidrule(lr){6-9}\cmidrule(lr){10-13}
       & $r=2$ &  $r=3$ & $r=4$ &  $r=5$ & $r=2$ &  $r=3$ & $r=4$ &  $r=5$ & $r=2$ &  $r=3$ & $r=4$ &  $r=5$ \\ \midrule
        Hics 10.1 & $2.0_{\pm 0.0}$ & -- & $4.0_{\pm 0.0}$ & -- & $2.0_{\pm 0.0}$ & -- & $4.0_{\pm 0.0}$ & -- & $2.0_{\pm 0.0}$ & -- & $4.0_{\pm 0.0}$ & -- \\ 
        Hics 20.1 & $2.0_{\pm 0.0}$ & $3.0_{\pm 0.0}$ & -- & $6.9_{\pm 0.4}$ & $2.0_{\pm 0.0}$ & $3.0_{\pm 0.0}$ & -- & $7.1_{\pm 0.3}$ & $2.0_{\pm 0.0}$ & $3.0_{\pm 0.0}$ & -- & $6.9_{\pm 0.3}$  \\ 
        Hics 30.1 & $2.0_{\pm 0.0}$ & $3.0_{\pm 0.0}$ & $4.5_{\pm 0.3}$ & $11.4_{\pm 0.9}$ & $2.0_{\pm 0.0}$ & $3.0_{\pm 0.0}$ & $4.8_{\pm 0.4}$ & $11.7_{\pm 1.3}$ & $2.0_{\pm 0.0}$ & $3.0_{\pm 0.0}$ & $6.0_{\pm 0.4}$ & $12.3_{\pm 1.6}$  \\ 
        Hics 40.1 & $2.0_{\pm 0.0}$ & $3.0_{\pm 0.0}$ & $6.2_{\pm 1.3}$ & $14.0_{\pm 1.1}$ & $2.0_{\pm 0.0}$ & $3.0_{\pm 0.0}$ & $7.4_{\pm 0.8}$ & $14.6_{\pm 0.7}$ & $2.0_{\pm 0.0}$ & $3.0_{\pm 0.0}$ & $8.0_{\pm 0.7}$ & $15.0_{\pm 1.0}$ \\ 
        Hics 50.1 & $2.0_{\pm 0.0}$ & $3.0_{\pm 0.0}$ & $8.5_{\pm 0.8}$ & $17.6_{\pm 1.6}$ & $2.0_{\pm 0.0}$ & $3.0_{\pm 0.0}$ & $11.5_{\pm 1.3}$ & $20.5_{\pm 1.7}$ & $2.0_{\pm 0.0}$ & $3.1_{\pm 0.2}$ & $14.3_{\pm 1.3}$ & $20.9_{\pm 0.7}$ \\ 
        Hics 75.1 & $2.0_{\pm 0.0}$ & $3.0_{\pm 0.0}$ & $18.1_{\pm 2.3}$ & $35.2_{\pm 2.9}$ & $2.0_{\pm 0.0}$ & $3.3_{\pm 0.5}$ & $19.7_{\pm 1.6}$ & $35.7_{\pm 2.6}$ & $2.0_{\pm 0.0}$ & $5.2_{\pm 1.2}$ & $26.4_{\pm 1.6}$ & $37.5_{\pm 1.5}$ \\
        Hics 100.1 & $2.0_{\pm 0.0}$ & $7.6_{\pm 2.6}$ & $41.0_{\pm 4.4}$ & $54.1_{\pm 3.6}$ & $2.0_{\pm 0.0}$ & $8.6_{\pm 2.5}$ & $41.1_{\pm 4.1}$ & $54.8_{\pm 2.0}$ & $2.0_{\pm 0.0}$ & $15.2_{\pm 1.6}$ & $43.8_{\pm 3.5}$ & $54.1_{\pm 3.5}$ \\
        \midrule\midrule
    \end{tabular}
  \end{threeparttable}
  }
  \label{tab:tix_hics}
\end{table}

Looking at \Cref{tab:tix_hics}, it is clear that TIX is able to find outlier subspaces of dimension $r = 2$ with no extra noise in all scenarios. The case $r=3$ is also perfectly recovered for any number of features $d\leq 50$, with minimal noise in higher dimensions if the refinement step of \Cref{alg:refine} is used. Outlier subspaces with $r= 4$ can be recovered with a few noisy features if $d\leq 50$, but on the other cases the amount of noise is considerably large. Finally, the case $r=5$ seems particularly challenging, and the results are only satisfactory for $d \leq 30$. This is because, for each subspace considered in \Cref{alg:tix}, TIX only computes the outlier score of the point of interest. Hence, it is possible that a combination of five irrelevant features produces a better outlier score in absolute value, and a comparison with the scores of other observations is needed to discern the actual outlier subspace.

On the other hand, the refinement step overall yields better results as we decrease $\beta$, even under similar computational constraints. Therefore, we suggest to decrease the value of $\beta$ instead of increasing $M$ in \Cref{alg:tix}.

\subsection{Empirical data}

Finally, we test the AIDA and TIX algorithms on some empirical data sets commonly used in the field of anomaly detection. The data sets are described in \Cref{tab:data_uci} in terms of the number of observations $n$, number of features $d$ and percentage of outliers. In some of the data sets, some preprocessing was required to define the outlier class. We refer to \cite{liu2012} for the definition of the outlier class in the Annthyroid, Arrhythmia, Breastw, ForestCover, Http, Ionosphere, Mammography, Pima, Satellite, Shuttle and Smtp data sets; and to \cite{aggarwal2015} for the Glass, Musk and Satimage-2 data sets. We refer to the same articles for information on how to obtain the data.

For consistency, we consider the same algorithms and configurations as we did at the beginning of \Cref{sec:hics}.

\begin{table}[ht]
  \caption{List of the empirical data sets we use in the analysis, together with some basic information about the number of observations, the dimensionality and the percentage of outliers.}
  \centering
  \begin{threeparttable}
    \begin{tabular}{lrrr}\midrule\midrule 
       & $n$ &  $d$ & \% outliers \\ \midrule
        Annthyroid & $6832$ & $6$ & $7$\\ 
        Arrhythmia & $452$ & $274$ & $15$\\ 
        Breastw & $683$ & $9$ & $35$\\
        ForestCover & $286048$ & $10$ & $0.9$\\
        Glass & $214$ & $9$ & $4.2$\\ 
        Http & $567497$ & $3$ & $0.4$\\ 
        Ionosphere & $351$ & $32$ & $36$\\ 
        Mammography & $11183$ & $6$ & $2$\\ 
        Musk & $3062$ & $166$ & $3.2$\\ 
        Pima & $768$ & $8$ & $35$\\ 
        Satellite & $6435$ & $36$ & $32$\\
        Satimage-2 & $5803$ & $36$ & $1.2$\\ 
        Shuttle & $49097$ & $9$ & $7$\\
        Smtp & $95156$ & $3$ & $0.03$\\ 
        \midrule\midrule
    \end{tabular}
  \end{threeparttable}
  \label{tab:data_uci}
\end{table}

\subsubsection{Anomaly detection in the empirical data sets}

The results are displayed in \Cref{tab:auc_uci}, where we present the performance of each model in terms of the AUC. The best two scores are highlighted in bold numbers. Similarly to the results of \Cref{tab:auc_hics}, the variance score function tends to perform better than the expectation score function, except in a few cases $(4/14)$. 

On the other hand, the randomized choice of $\alpha$ has a smaller impact in \Cref{tab:auc_uci} compared to \Cref{tab:auc_hics}. There are two possible explanations for these differences. One of them is that we are randomizing $\alpha$ in an interval with opposing effects: $\alpha >1$ enlarges the intervals in \Cref{eq:var_expo}, while $\alpha < 1$ shrinks them. Hence the average corresponds to a low-risk/low-reward ensemble that dilutes these effects. Using ensembles with only $\alpha < 1$ or only $\alpha > 1$ could be an alternative to explore the benefits of \Cref{eq:var_expo} over \Cref{eq:var_uni} in that case. 

The other possible explanation is that most of the detected outliers in \Cref{tab:auc_uci} are strong outliers, since slight variations in the value of $\alpha$ do not have a large impact on the score of strong outliers. This second explanation seems more plausible in this case, considering that the HiCs data sets do not have many strong outliers.

\begin{table}[ht]
  \ra{1.3}
  \caption{AUC obtained in the empirical data sets with the different anomaly detection models. The variants of the AIDA algorithm are labelled according to the score function used---Expectation (E) or Variance (V)---and the choice of $\alpha$ in \Cref{eq:mean_expo,eq:var_expo} ($\alpha = 1$ (1) or random choice (R)).}
  \centering
  \resizebox{\columnwidth}{!}{
  \begin{threeparttable}
    \begin{tabular}{lcccccccc}\midrule\midrule
       & AIDA (E1) &  AIDA (ER) & AIDA (V1) &  AIDA (VR) & iForest & iNNE & LOF & AvgKNN\\ \midrule
        Annthyroid & $\bm{0.823_{\pm 0.011}}$ & $\bm{0.817_{\pm 0.013}}$ & $0.809_{\pm 0.009}$ & $0.814_{\pm 0.008}$ & $0.809_{\pm 0.012}$ & $0.699_{\pm 0.010}$ & $0.744$ & $0.807$ \\ 
        Arrhythmia & $0.784_{\pm 0.008}$ & $0.784_{\pm 0.008}$ & $0.798_{\pm 0.001}$ & $\bm{0.800_{\pm 0.002}}$ & $\bm{0.804_{\pm 0.013}}$ & $0.753_{\pm 0.007}$ & $0.796$ & $0.776$  \\ 
        Breastw & $0.980_{\pm 0.001}$ & $0.981_{\pm 0.002}$ & $0.981_{\pm 0.002}$ & $0.982_{\pm 0.001}$ & $\bm{0.986_{\pm 0.002}}$ & $0.724_{\pm 0.029}$ & $0.384$ & $\bm{0.986}$  \\ 
        ForestCover & $0.857_{\pm 0.015}$ & $0.854_{\pm 0.012}$ & $0.861_{\pm 0.016}$ & $0.865_{\pm 0.011}$ & $\bm{0.876_{\pm 0.019}}$ & $\bm{0.955_{\pm 0.009}}$ & $0.536$ & $0.790$  \\ 
        Glass & $0.885_{\pm 0.005}$ & $0.886_{\pm 0.006}$ & $\bm{0.894_{\pm 0.006}}$ & $\bm{0.894_{\pm 0.004}}$ & $0.811_{\pm 0.006}$ & $0.872_{\pm 0.015}$ & $0.830$ & $\bm{0.903}$  \\ 
        Http  & $0.994_{\pm 0.000}$ & $0.994_{\pm 0.001}$ & $\bm{0.998_{\pm 0.001}}$ & $0.996_{\pm 0.000}$ & $\bm{1.000_{\pm 0.000}}$ & $\bm{0.998_{\pm 0.002}}$ & $0.352$ & $0.133$  \\
        Ionosphere & $0.912_{\pm 0.003}$ & $0.914_{\pm 0.003}$ & $0.921_{\pm 0.002}$ & $\bm{0.923_{\pm 0.002}}$ & $0.860_{\pm 0.005}$ & $0.901_{\pm 0.008}$ & $0.840$ & $\bm{0.934}$  \\ 
        Mammography & $\bm{0.858_{\pm 0.006}}$ & $0.857_{\pm 0.008}$ & $0.856_{\pm 0.007}$ & $0.852_{\pm 0.008}$ & $\bm{0.859_{\pm 0.008}}$ & $0.825_{\pm 0.011}$ & $0.719$ & $0.849$  \\ 
        Musk & $0.978_{\pm 0.013}$ & $0.994_{\pm 0.005}$ & $\bm{1.000_{\pm 0.000}}$ & $\bm{1.000_{\pm 0.000}}$ & $0.999_{\pm 0.001}$ & $\bm{1.000_{\pm 0.000}}$ & $0.453$ & $0.826$  \\
        Pima & $0.702_{\pm 0.006}$ & $0.699_{\pm 0.006}$ & $\bm{0.714_{\pm 0.004}}$ & $0.713_{\pm 0.006}$ & $0.675_{\pm 0.013}$ & $0.684_{\pm 0.006}$ & $0.621$ & $\bm{0.714}$  \\ 
        Satellite & $0.717_{\pm 0.004}$ & $0.721_{\pm 0.004}$ & $\bm{0.746_{\pm 0.004}}$ & $\bm{0.751_{\pm 0.004}}$ & $0.704_{\pm 0.015}$ & $0.739_{\pm 0.016}$ & $0.553$ & $0.689$  \\
        Satimage-2 & $0.997_{\pm 0.001}$ & $0.998_{\pm 0.001}$ & $\bm{0.999_{\pm 0.000}}$ & $\bm{0.999_{\pm 0.001}}$ & $0.993_{\pm 0.001}$ & $0.997_{\pm 0.001}$ & $0.537$ & $0.966$  \\
        Shuttle & $0.967_{\pm 0.004}$ & $0.968_{\pm 0.006}$ & $0.983_{\pm 0.001}$ & $\bm{0.985_{\pm 0.001}}$ & $\bm{0.994_{\pm 0.001}}$ & $\bm{0.985_{\pm 0.004}}$ & $0.539$ & $0.687$  \\ 
        Smtp  & $0.904_{\pm 0.001}$ & $\bm{0.907_{\pm 0.002}}$ & $0.899_{\pm 0.002}$ & $0.898_{\pm 0.002}$ & $0.879_{\pm 0.008}$ & $\bm{0.909_{\pm 0.009}}$ & $0.441$ & $0.906$  \\ 
        \midrule\midrule
    \end{tabular}
  \end{threeparttable}
  }
  \label{tab:auc_uci}
\end{table}

Moreover, from \Cref{tab:auc_uci} it is clear that AIDA performs favourably compared to other state-of-the-art methods. In particular, algorithms whose performance is highly dependent on the choice of certain parameters (i.e., iNNE, LOF and AvgKNN) can perform very well on some data sets but poorly on others. In contrast, all variations of AIDA are very stable, sometimes yielding the best or second-best results (12/14). Furthermore, in the data sets where it does not give the highest AUC, the difference is usually very small (around 0.01 AUC), except in one case (ForestCover).

\subsubsection{Anomaly explanation in the empirical data sets}

Even though these are labelled data sets, and we know beforehand which are the potential outliers, we do not know which features caused them. Hence, we cannot do the same comparison as we did in \Cref{sec:hics}. As an alternative, we choose the data set with the highest dimensionality in \Cref{tab:data_uci}---i.e. the Arrhythmia data set---and analyze some of the most anomalous points classified by AIDA (VR), i.e, the AIDA model with variance score and random $\alpha$. Doing so will allow us to verify whether the labelled outliers are actually the only points that look different from the rest of the data set, or if this does not hold for some of them.

Specifically, a data set should verify two conditions to qualify as a good benchmark for anomaly detection: the labelled points must be different in some way from the rest of the data set, while unlabelled points should look ``normal'' (or not unusual) in any feature subspace. Empirical data sets often show anomalies that are caused by a particular mechanism, so that outliers generated by other causes end up being classified as inliers. However, anomaly detection methods do not make distinctions with respect to the types of outliers, but only consider whether a point is anomalous or not. Thus, empirical data sets may give the impression that a certain algorithm is not performing well, because the data set targets a particular type of outlier, while most algorithm do not. We illustrate this problem using anomaly explanation and show that some points are indeed anomalous even though they are not labelled as outliers. The setting for the TIX algorithm is $M=1$, $k_{min} = 10$ and $\beta = 1.5$ (see \Cref{alg:tix,alg:refine}), and the results were repeated 10 times for consistency.

In \Cref{fig:tix_arrhythmia}, we present the explanation results of the first and fourth most outlying points detected by AIDA (VR)---in fact, these points were signalled as outliers by all methods considered in this paper. We have chosen these two points, because the first is an actual outlier, while the second is the most anomalous observation, according to AIDA, that was labelled as in inlier in the original data set. 

The DPPs for each of the two points using the 20 most relevant features are shown in the top row of \Cref{fig:tix_arrhythmia}, which immediately illustrates that these points are easily isolated with respect to their most relevant features. This can be visualized in the lower row of \Cref{fig:tix_arrhythmia}, where we present 2D plots of the two most relevant features for each observation. The features are numbered in the order they appear in the original data set (from 0 to 273), once the features with missing values have been removed.

It is remarkable that the shape of the 2D plots is very similar in both cases, with the points of interest lying on the opposite side of the majority of the data set, contained in the origin $(0,0)$ in both plots. Moreover, in the lower right plot of \Cref{fig:tix_arrhythmia} we observe another labelled outlier that is also easy to isolate from most of the inliers in that feature combination. 

Interestingly, there is another observation close to this labelled outlier that was classified as an inlier. We note that both points were also reported as outliers by all the algorithms considered here. Therefore, we conclude that, while this data set contains labelled outliers that are indeed anomalous, there exist also labelled inliers with similar outlying properties. 

The consequence is that the performance reported by the anomaly detection algorithms could be low, not because they do not detect the anomalous points, but rather because they do not detect outliers of a particular kind. Explanation methods, such as the TIX algorithm proposed here, can help determining whether anomalous points are caused by the relevant mechanisms of a particular application by analyzing the most relevant features that explain each outlier. 

\begin{figure}[ht]
    \centering
    \begin{subfigure}{0.45\textwidth}
    \centering
    \includegraphics[width=\textwidth]{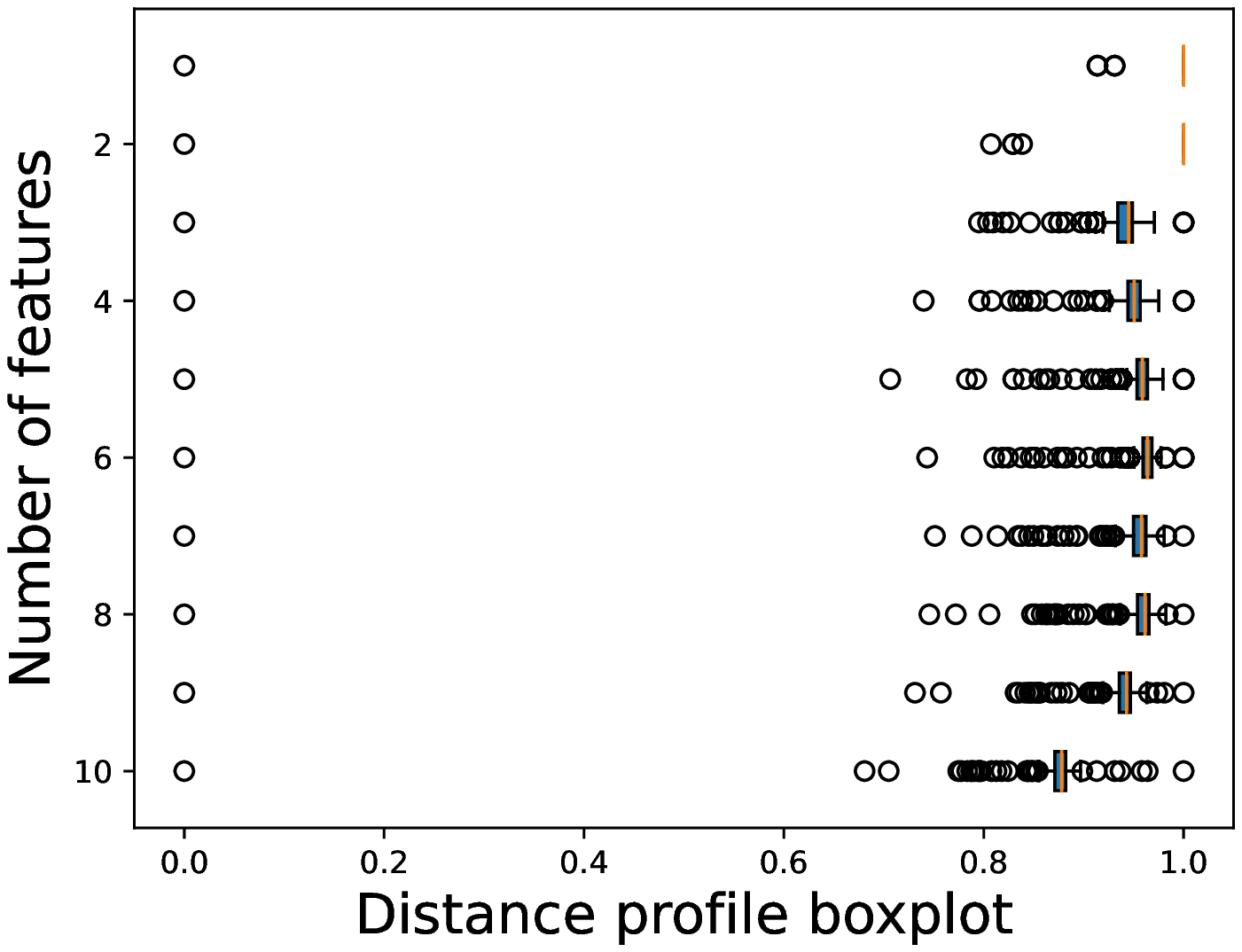}
    \caption{DPP of the labelled outlier.}
    \end{subfigure}    
    \begin{subfigure}{0.45\textwidth}
    \centering
    \includegraphics[width=\textwidth]{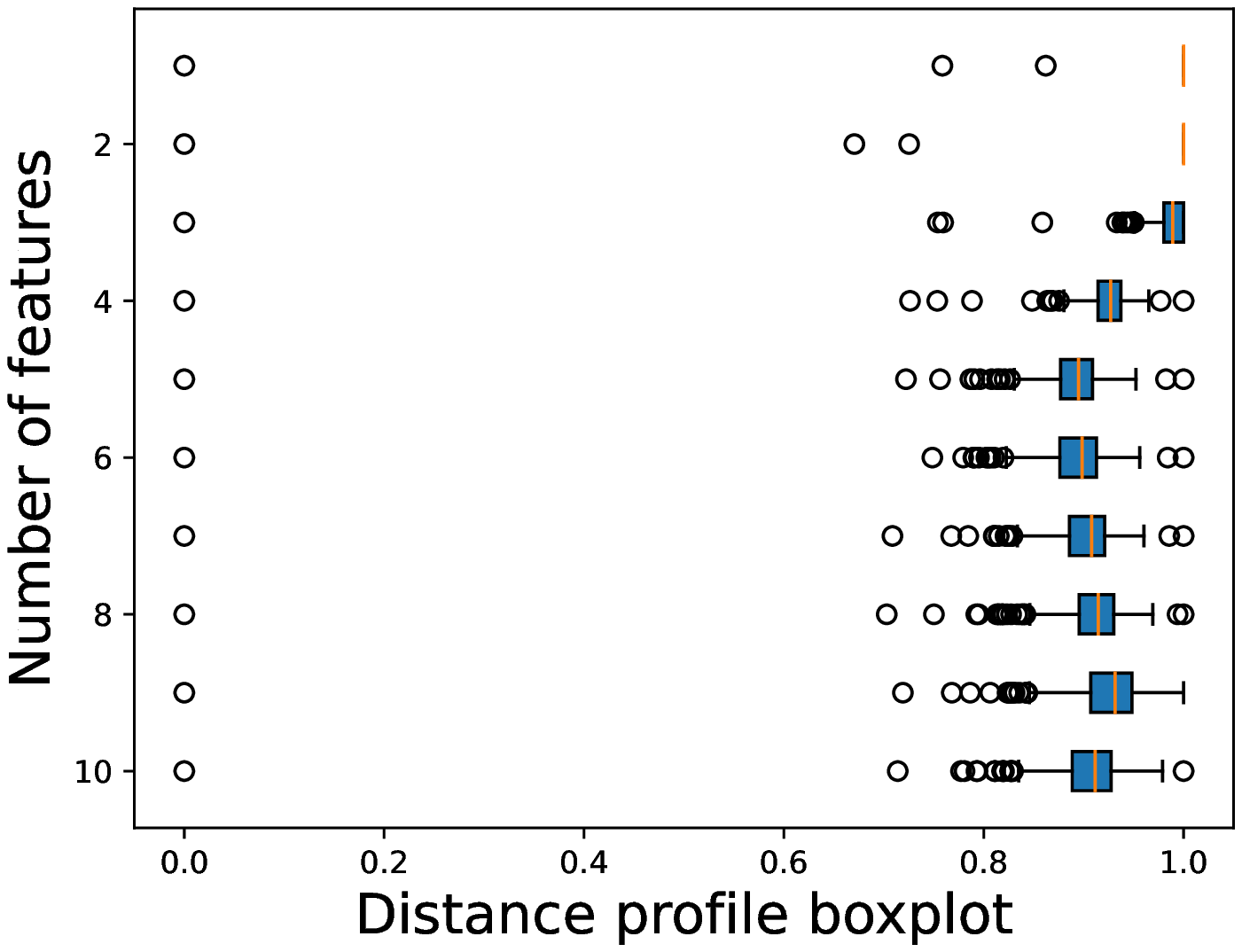}
    \caption{DPP of the labelled inlier.}
    \label{fig:arrhythmia_inlier_dpp}
    \end{subfigure}
    \medskip
    \begin{subfigure}{0.45\textwidth}
    \centering
    \includegraphics[width=\textwidth]{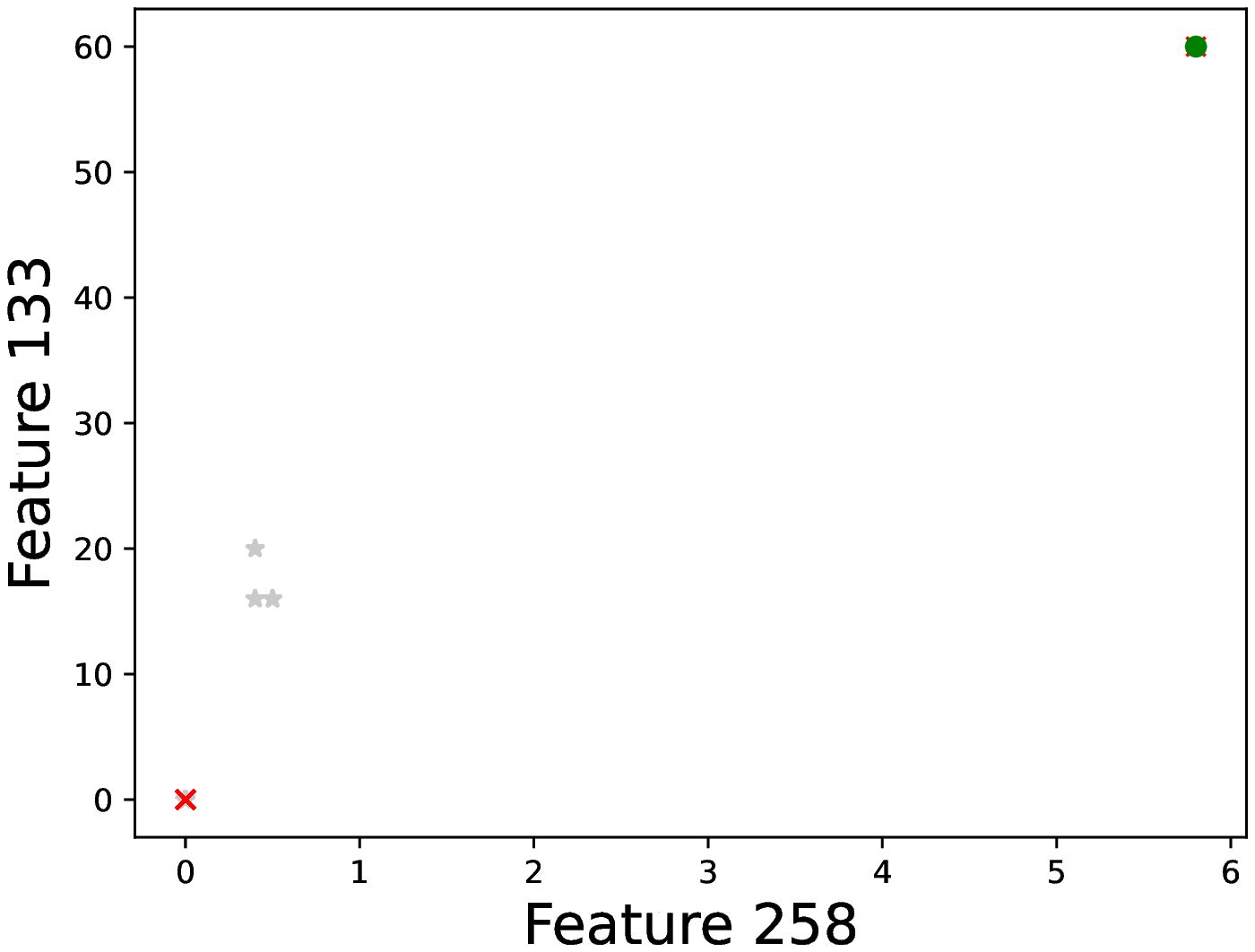}
    \caption{2D plot of the outlier.}
    \end{subfigure}    
    \begin{subfigure}{0.45\textwidth}
    \centering
    \includegraphics[width=\textwidth]{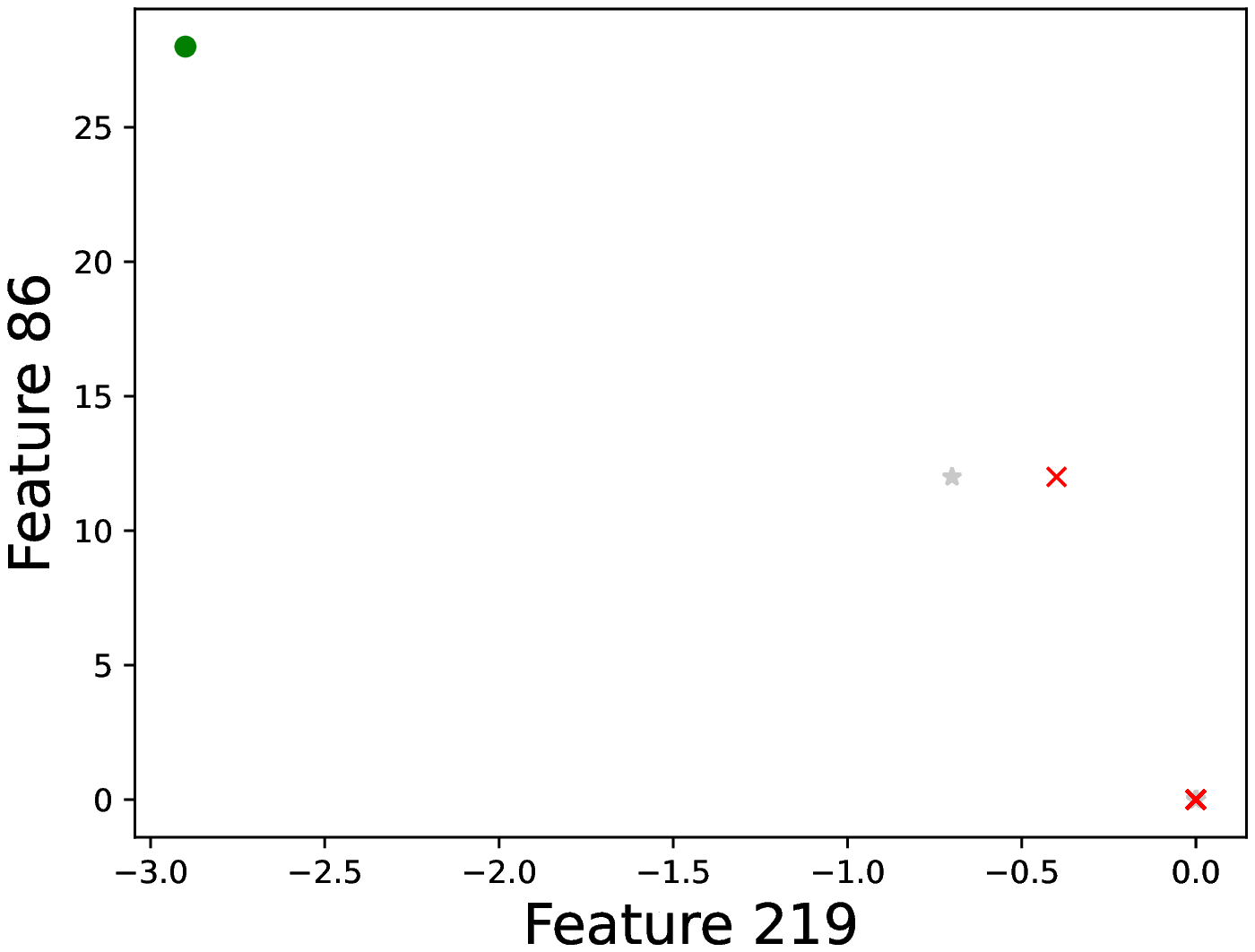}
    \caption{2D plot of the inlier.}
    \label{fig:arrhythmia_inlier_plot}
    \end{subfigure}

    \caption{Analysis of two observations of the Arrhythmia data set classified as outliers by AIDA. The top row consists of the DPPs of an actual outlier (left) and a point originally classified as inlier (right). The lower row contains 2D plots of the two most relevant features for each point, marked as green dots. Inliers are marked with gray starts and outliers with red crosses.}
    \label{fig:tix_arrhythmia}
\end{figure}

\section{Conclusions}\label{sec:conclusions}

In this paper, we have proposed two new algorithms, under the acronyms AIDA (Analytic Isolation and Distance-based Anomaly) and TIX (Tempered Isolation-based Explanation).

The anomaly detection algorithm AIDA has been shown to generate similar or superior performances when compared to other state-of-the-art algorithms, especially in the case of multidimensional outlier hidden subspaces. This is partially due to the definition of outlier employed by AIDA, which inclines towards points that can be easily isolated, regardless of whether those points belong to extreme or interior values, in contrast to the artificial regions created by iForest. We have also proved several results concerning isolation methods, such as analytical formulas for the moment generating function and the first two cumulants of the number of random splits, and the convergence rate of the probability that the iForest algorithm \cite{liu2012} finds specific feature subspaces of a given dimensionality.

In discussing the TIX algorithm, we have shown that it provides accurate explanations for outliers hiding in two- and three-dimensional subspaces, even when the number of irrelevant features is extremely large. 

Moreover, the DPP (distance profile plot) has been proposed as a visualization tool, which, in combination with the traditional 2D plots, can immediately find subspaces where the outliers can be isolated. This has been illustrated using empirical data sets with hundreds of features, and it has been shown that the explanations can be useful to filter anomalous points generated by different mechanisms.

\appendix

\bibliography{references.bib}
\bibliographystyle{myplainnat} 

\section{iForest and hidden subspaces}\label{sec:app_proof}

We prove that, if the anomalies are hidden in multidimensional subspaces of size $r$, the probability that an isolation tree finds such subspace decays as $\mathcal{O}(d^{-r})$, where $d$ is the number of features. Therefore, iForest is not a suitable algorithm for detecting outliers hidden in multidimensional subspaces.

Let $\bm{X}_n$ be a data set of size $n$ and dimensionality $d$ containing an outlier $X_o$ in a unique feature subspace of dimensionality $r$. Furthermore, assume that $X_o$ cannot be distinguished from an inlier in any lower feature subspace of size smaller than $r$. For simplicity, also assume that $X_o$ is a strong outlier in the hidden subspace, such that when this subspace is found, $X_o$ is easily detected as an outlier. The following proposition gives a recursion formula to compute the probability of finding such subspace.

\begin{prop}\label{prop:hidden_prob}
Let $d,r,h_{M} \in \mathbb{N}^+$, with $r\leq d$, be the number of features, size of the hidden subspace, and maximum depth of an isolation tree, respectively. Denote by $p(r,h_{M})$ the probability that a subspace of size $r$ is found by an isolation tree of length $h_{M}$. Then, $p(r,h_{M})$ admits the following recursion formula:
\begin{equation}\label{eq:hidden_subr}
    p(r,h_{M}) = \frac{r}{d}\sum_{i=1}^{h_{M}} \left(1-\frac{r}{d}\right)^{i-1}p(r-1,h_{M}-i),
\end{equation}
\begin{equation}\label{eq:hidden_sub1}
    p(1,h_{M}) = 1-\left(1-\frac{1}{d}\right)^{h_{M}}.
\end{equation}
\end{prop}
\begin{proof}
\Cref{eq:hidden_sub1} is simply the complementary of the probability of not selecting a particular feature in $h_{M}$ steps. \\
\Cref{eq:hidden_subr} follows from the fact that, if it takes $i$ steps to randomly select one of the features belonging to the hidden subspace, the problem can be reduced to finding the remaining $r-1$ features in $h_{M}-i$ steps. \\
The total probability is thus the summation over all these combinations, where $\frac{r}{d} (1-\frac{r}{d})^{i-1}$ is the probability that it takes $i$ steps to select one of the relevant features.
\end{proof}

From the results of \Cref{prop:hidden_prob}, it is now easy to prove that $p(r,h_{M})$ decays as $\mathcal{O}(d^{-r})$ for large $d$. In particular, a simple Taylor expansion shows that $p(1,M) \approx h_{M}/d$ for large $d$. Plugging this result into \Cref{eq:hidden_subr} gives the aforementioned convergence rate of $p(r,h_{M})$ towards zero. 

\begin{remark}
While a large value of $h_{M}$ would help in reducing the convergence rate, this is not possible in data sets where the number of features is equal or higher than the number of observations, as the maximum depth of an isolation tree without pruning is $n-1$ \cite{liu2012}. Furthermore, iForest extracts most of the outlier information during the first splits, therefore a large $h_{M}$ provides little additional information to distinguish outliers from inliers.
\end{remark}

\end{document}